\documentclass[preprint,12pt]{elsarticle}

\usepackage{ulem}
\usepackage{amssymb}
\usepackage{amsmath}
\usepackage{amsthm,amsmath,amssymb}
\newtheorem{myDef}{Definition}
\usepackage{subfigure}
\usepackage[colorlinks,urlcolor=blue,linkcolor=blue,citecolor=blue]{hyperref}
\usepackage[ruled,linesnumbered]{algorithm2e}
\usepackage{multirow}
\usepackage{makecell}
\newtheorem{proposition}{Proposition}
\newtheorem{definition}{Definition}
\usepackage{color}
\usepackage{xcolor}  
\definecolor{mbycolor}{RGB}{128,10,80}  
  
\journal{Artificial Intelligence}

\begin{document}

\begin{frontmatter}



\title{Active Legibility in Multiagent Reinforcement Learning}


\author{Yanyu Liu} 
\affiliation{organization={School of Automation, Central South University},
addressline={No.605 South Lushan Road}, 
city={Changsha},
postcode={410083},
state={Hunan},
country={China}}

\author{Yinghui Pan}
\affiliation{organization={National Engineering Laboratory for Big Data System Computing Technology, Shenzhen University},
city={Shenzhen},
state={Guangdong},
country={China}}

\author{Yifeng Zeng$^*$}
\affiliation{organization={Department of Computer and Information Sciences, Northumbria University},
city={Newcastle},
country={United Kingdom}}
\author{Biyang Ma}
\affiliation{organization={School of Computer Science, Minnan Normal University},
city={Zhangzhou},
state={Fujian},
country={China}
}

\author{Doshi Prashant}
\affiliation{organization={Department of Computer Science, University of Georgia},
city={Athens},
state={Georgia},
country={USA}}

\begin{abstract}

A multiagent sequential decision problem has been seen in many critical applications including
urban transportation, autonomous driving cars, military operations, etc.
Its widely known solution, namely multiagent reinforcement learning, has evolved tremendously in recent years. 
Among them, the solution paradigm of modeling other agents attracts our interest, which is different from traditional value decomposition or communication mechanisms. 
It enables agents to understand and anticipate others' behaviors and facilitates their collaboration.
Inspired by recent research on the legibility that allows agents to reveal their intentions through their behavior, we propose a {\it multiagent active legibility framework} to improve their performance.
 The legibility-oriented framework allows agents to conduct legible actions so as to help others optimise their behaviors. 
In addition, we design a series of problem domains that emulate a common scenario and best characterize the legibility in multiagent reinforcement learning.
The experimental results demonstrate that the new framework is more efficient and costs less training time compared to several multiagent reinforcement learning algorithms. 
\end{abstract}


\begin{highlights}
\item Research highlights 1. We propose a Legible Interactive-POMDP~(LI-POMDP) as well as a Multi-Agent Active Legibility (MAAL) framework, which enhances the legibility of individual agents' actions, thereby improving collaboration in multi-agent reinforcement learning environments.
\item Research highlights 2. We develop a series of problem domains that emphasize the speed and accuracy of goal recognition and cooperation between
agents, effectively showcasing the importance and necessity of legibility in multi-agent reinforcement learning.

\end{highlights}

\begin{keyword}
Legibility \sep Multiagent reinforcement learning \sep Multiagent interaction


\end{keyword}

\end{frontmatter}



\section{Introduction}
\label{sec:Introduction}

Multiagent Reinforcement Learning~(MARL), as a powerful method to tackle multiagent sequential decision problems, has grown tremendously in the recent two decades~\cite{tan1993multi,littman1994markov}. 
In the MARL research development, cooperative tasks among agents have emerged as one of the main focuses. Enabling agents to learn cooperative behaviors facilitates the completion of more complex tasks, thereby providing greater benefit to human life.
When MARL evolves to learn agents' collaboration, it often leads to two branches: value decomposition and centralized-critic. The value decomposition methods train a global Q-network with the consideration of global information and are able to overcome the MARL instability, e.g. VDN~\cite{sunehag2017value}, QTRAN~\cite{son2019qtran}, QMIX~\cite{rashid2020weighted}, etc. The centralized-critic methods aim to learn a centralized critic network, and then use it to train distributed actor policy, e.g. MADDPG~\cite{lowe2017multi}, COMA~\cite{foerster2018counterfactual}. 
Beyond this, communication has been adopted in MARL  to transmit local information, which can be aggregated into a global perspective.
Foerster {\it et al.}~\cite{foerster2016learning} first investigated the autonomous communication among agents in the learning. Many other works~\cite{sukhbaatar2016learning,peng2017multiagent,wang2020learning,yuan2022multi} have contributed to the communication in MARL. 
The communication mechanisms, however, encounter challenges, e.g. communication interference and noise, bandwidth limitations, etc. in practical engineering applications.

Recently, a new framework appears by modeling other agents in MARL~\cite{gmytrasiewicz2005framework,he2016opponent}. This framework assumes that it is beneficial for each agent to account for how the others would react to its future behaviors. Wen {\it et al.}~\cite{wen2019probabilistic} presented a probabilistic recursive reasoning~(PR2) framework, which adopted variational Bayes methods to approximate the opponents’ policy, to which each agent found the best response and then improved its own policy. 

Empirically, predicting the intended end-product~(goal), as well as the sequence of steps (plan), can be helpful for improving the performance of agents' interaction in MARL~\cite{albrecht2018autonomous,dann2023multi,zhang2023intention}.
In MARL, agents could model the intentions and policies of their teammates, allowing them to better coordinate their actions and optimize their collective performance. 
By recognizing others' plans, the agents can adapt their own behaviors to support or complement their teammates’ goals therefore leading to more efficient and effective collaboration. 
However, solely relying on the modeling capabilities of the agents has limitations and faces numerous challenges in enhancing system performance.
One of the challenges is {\it ambiguity}: extracting the plan or predicting the future actions from observed trajectories could be confusing and complex, as trajectories may have multiple interpretations or could be misinterpreted by unforeseen factors. 
Once the agent recognizes the incorrect intention, the consequence could be devastating and catastrophic under some circumstances. On the other hand, the contradiction between the recognition accuracy and computational complexity severely limits the generalization of plan recognition in MARL. 

Recently, {\it legibility} has been introduced to facilitate agents to convey their intentions through their behaviors~\cite{miura2021maximizing}. 
In other words, the legibility could reduce the ambiguity over the possible goals of agents from an observer's perspective and improve human-and-agent collaboration.
The most recent work~\cite{liu2023improvement}, namely Information Gain Legibility~(IGL), used the information entropy gain to shape the reward signal and improved the {\it agent-to-human} policy legibility. Compared to the existing methods~\cite{dragan2013generating,dragan2013legibility}, IGL is much easier to implement and can theoretically be extended to many reinforcement learning methods.
Thus, we will investigate the MARL legibility further in this paper.
\begin{figure}[h]
\centering
    \subfigure[One sample of the illegible route]{\includegraphics[width=0.4\linewidth]{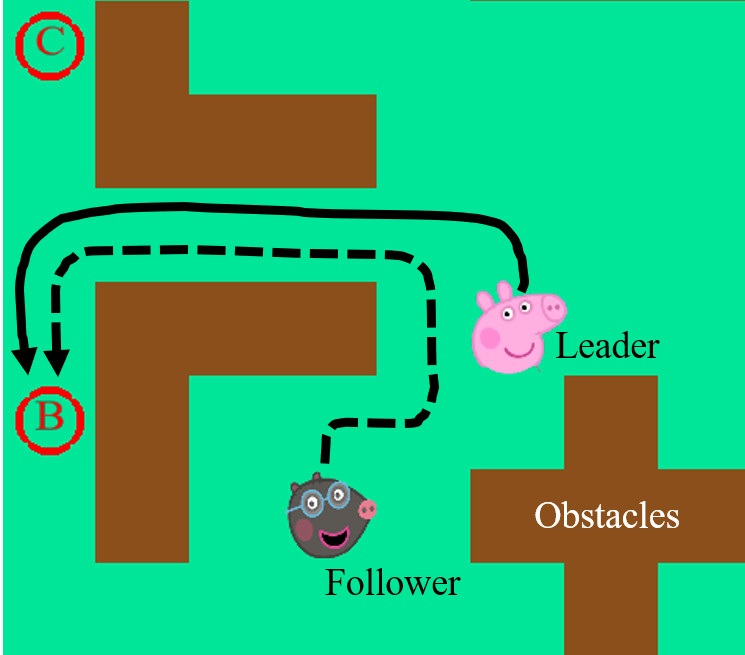}
    \label{fig:unlegible}}
    \subfigure[One sample of the legible route]{\includegraphics[width=0.4\linewidth]{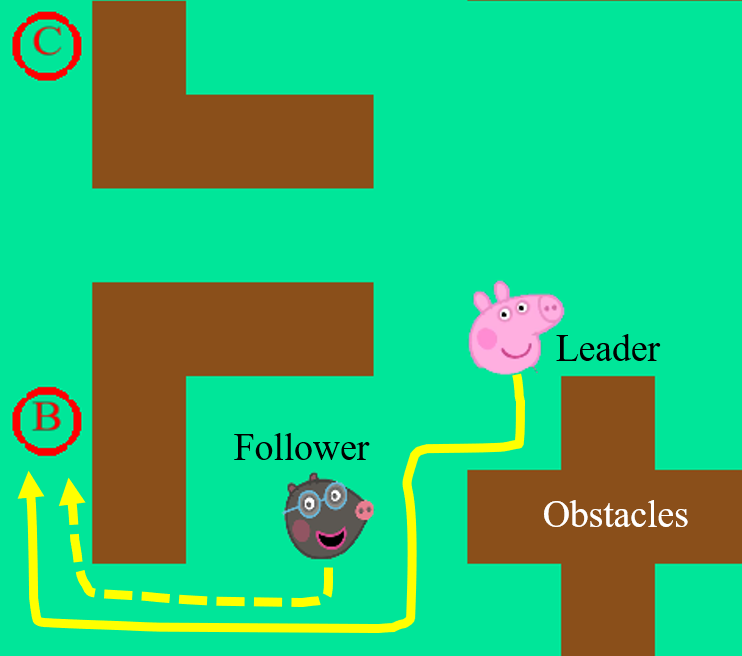}
    \label{fig:legible}}
    \caption{Two trajectories are executed by the two agents~(one leader and its follower) who aim to complete the grasping task}
    \label{fig:why-need-legibility}
\end{figure}

We begin with a toy example to show how  the legibility facilitates multiagent collaboration in Fig.~\ref{fig:why-need-legibility}. 
Two agents~(one leader and its follower) are rendered in a grid world and aim to meet at one of the exits (red circle with a letter). Only the leader knows which exit is the target, and the follower must identify the target from the leader's trajectory and get there. 
Fig.~\ref{fig:unlegible} demonstrates one way of completing the task, where the leader chooses the shortest route~(solid black line) to its target $B$. However, from the follower's perspective, the leader has an equal chance of going to $B$ and $C$ in the first five steps. With this confusion, the follower has to take a longer detour to tailgate the leader~(dashed black line), and the two agents take eighteen moves in total eventually. 
In contrast, Fig.~\ref{fig:legible} shows that the navigator sacrifices the optimality, choosing a longer route~(solid yellow line) to improve the legibility, helping the follower identify the target earlier and take a shortcut. In this case, by utilizing the legibility, the leader agent ultimately reduces the total steps into sixteen moves therefore increasing their collaboration.

{
The example shows that for some specialized multiagent tasks, introducing the legibility can reduce the ambiguity in agents' modeling, providing advantages that cannot be achieved solely by improving agents' learning capabilities, which raises the upper limit of the talk completion performance. 
Additionally, the legibility can also increase the RL policy transparency and interpretability, which is highly significant for safety-critical domains such as military, healthcare, finance, etc~\cite{abouelyazid2024reinforcement,charpentier2021reinforcement}.}

{Driven by the aforementioned concepts, we contribute a novel approach to multiagent reinforcement learning by conducting the first exploration on the MARL legibility.
In this paper, we present a Multiagent Active Legibility~(MAAL) framework that exploits the legibility in MARL through individual agents' perspectives.
In MAAL, a subject agent adapts a reward shaping technique to conduct legible actions by narrowing down the margin between the predicted actions from other agents and the true ones to be executed by the subject agent.
Compared with the previous IGL research, MAAL improves the {\it agent-to-observer} legibility in a RL agent and seeks to achieve seamless cooperation between the agent and observer.
Furthermore, unlike IGL which leverages information entropy of the observer's beliefs to shape the reward function, MAAL uses the variation in the Kullback–Leibler~(KL) divergence between the observer's beliefs and the true goal before and after executing their actions, and instructs the agent to derive the legibility reward function. 

The main contributions of this paper are as follows: 

\begin{itemize}
    \item We define a Legible Interactive-POMDP~(LI-POMDP) and propose a Multi-Agent Active Legibility (MAAL) framework, which enhances the legibility of individual agents' actions, thereby improving collaboration in multi-agent reinforcement learning environments.
    \item We design a series of problem domains that simulate typical scenarios, effectively showcasing the importance of legibility in multi-agent reinforcement learning.
    \item We conduct extensive comparative experiments, ablation studies, and theoretical analyses, demonstrating the convergence and effectiveness of the MAAL framework.
\end{itemize}
The remainder of this paper is organized as follows: Section~\ref{sec_related} examines relevant representative research. Section~\ref{Preliminary} introduces the background knowledge pertaining to multiagent decision-making and the policy legibility. We present the new MAAL framework in Section \ref{sec_method}. Section~\ref{canalysis} offers a precise convergence analysis of the legibility algorithm. The experimental settings and discussion on simulation results are presented in Section \ref{sec_exp}. Lastly, we summarize this research and brief the future work in Section \ref{sec_conslusion}.
}

\section{Related Works}
\label{sec_related}
In this section, we will review the relevant research on three aspects: MARL, modelling other agents and legibility. Referring to multiagent reinforcement learning, there has seen a series of classical methods. Sunehag {\it et al.}~\cite{sunehag2017value} proposed the Value-Decomposition Networks~(VDN), which views the system's overall value as the summation of the individual value of agents. VDN substantially 
 overcomes the non-stationarity of MARL in a concise way. However, the linear summation may not accurately represent the overall value in some complex environments. QMIX~\cite{rashid2020weighted} inherits the ideology and proposes to approximate a monotonic function between the local Q-network and the overall Q-network with the nonlinear neural network. Son {\it et al.}~\cite{son2019qtran} proposed QTRAN to decompose the joint Q-value function into a sum of task-specific value functions and used task-relational matrices to transfer knowledge between the tasks - being free from the subjective {\it Additivity} and {\it Monotonicity}. 
 
 In parallel, a lost of research expands the policy gradient method to MARL applications. Lowe {\it et al.}~\cite{lowe2017multi} proposed multiagent Deep Deterministic Policy Gradient~(MADDPG), where agents learned to estimate their own action-value function by utilizing a centralized critic network that took into account the joint actions of all the agents. Then Foerster {\it et al.}~\cite{foerster2018counterfactual} introduced Counterfactual Multiagent~(COMA) that used the counterfactual baseline, which attempted to estimate what the value would have been with alternative joint actions. This research has almost dominated the MARL development in recent years, and has achieved promising results in solving complex problems such as {\it StarCraft Multiagent Challenge}. 

In human cooperation activities, persons always subconsciously predicts others' next moves, or intentions to rectify their own policy. Gmytrasiewicz {\it et al.}~\cite{gmytrasiewicz2005framework} proposed interactive POMDP~(I-POMDP) that introduced models of other agents into the partially observed MDP. In I-POMDP, an agent is able to construct belief models regarding its understanding of the knowledge and beliefs held by other agents. However, solving the I-POMDP is rather difficult due to the inherent complexity even with purpose-designed methods~\cite{doshi2008generalized,doshi2009monte,sonu2015scalable}. 
Wen {\it et al.}~\cite{wen2019probabilistic} proposed Probabilistic Recursive Reasoning~(PR2) which implemented the recursive reasoning in MDP, and encouraged the modeling method from game theory\cite{doshi2010modeling,fotiadis2022recursive}. Later, they proposed Generalized Recursive Reasoning~(GR2)~\cite{wen2019modelling} which took bounded rationality into the consideration and extended the reasoning level to an arbitrary number. However, the inherent computational and action complexity of RR limits its application and, in some scenarios, the recursive reasoning can even be an NP-hard problem. On the contrary, plan recognition is not only more computation-friendly, but less reliant on a complete and accurate model of beliefs, preferences, and actions, as well as mutual beliefs and knowledge among other agents~\cite{albrecht2018autonomous}.

The existing work has shown the impact of legibility on agent's policies.
Dragan {\it et al.}~\cite{dragan2013legibility} presented the mathematical definition of legibility, and proposed a model in which agent could evaluate or generate motion by the functional gradient optimization in the space trajectory~.
Holladay {\it et al.}~\cite{holladay2014legible} studied the implication of legibility in robot pointing, e.g. producing pointing gestures.
Nikolaidis {\it et al}.~\cite{nikolaidis2016based} explored the impact of the observer's viewpoint on the legibility of motion and trajectories and proposed a viewpoint-based strategy for optimizing legibility and offered new insights into using occlusion to generate deliberately ambiguous or deceptive movements.
Bied {\it et al.}~\cite{bied2020integrating} proposed a reward-shaping method that embeds an observer model within reinforcement learning to enhance the legibility of an agent's trajectory.
Busch {\it et al.}~\cite{busch2017learning} introduced a model-free reinforcement learning algorithm to optimize a general, task-independent cost function, together with an evaluation strategy to determine whether the learned behavior is universally legible.
Persiani {\it et al.}~\cite{persiani2023policy} proposed a method that utilizes Bayesian Networks to model both the agent and the observer. They optimized the legibility of the agent model by minimizing the cross-entropy between these two networks. 
Bernardini {\it et al.}~\cite{bernardini2024optimizing}  presents a joint design for goal legibility and recognition in a cooperative, multi-agent scenario with partial observability. 
Zhao {\it et al.}~\cite{zhao2020actor} conducted legibility tests within the context of Deep Reinforcement Learning (DRL) and suggested the use of Recurrent Neural Networks (RNN) as motion predictors to score the legibility of actions.
Miura {\it et al.}~\cite{miura2020maximizing,miura2021maximizing} extended the work on maximizing legibility from deterministic settings to stochastic environments, and introduced a Legible Markov decision problem~(L-MDP).
They also proposed Observer-Aware MDP~(OAMDP)~\cite{miura2021unifying} as a less complex framework of I-POMDP under certain conditions and extended it with explicit communication as Communicative Observer-Aware MDP (Com-OAMDP)\cite{miura2024observer}.
Faria {\it et al.}~\cite{faria2024guess} proposed a more computation-friendly framework compared to L-MDP, namely Policy Legible MDP~(PoLMDP), which considerably lowered the complexity by solving nearly a standard MDP.
This work effectively enriches the agent's ability to convey intentions to human observers through actions  and achieves the smooth human-and-robot cooperation. 

\section{Preliminary}
\label{Preliminary}
\subsection{Markov Decision Process and I-POMDP}
The problem of learning a goal-directed policy is typically abstracted into a mathematical framework known as Markov Decision Process~(MDP), which is a discrete-time stochastic process with the Markovian property. 
In MDP, the future state depends solely on the current state and the action taken, independent of past states and actions. 
In this paper, we define MDP as a tuple $\mathcal{M}=\left\{ \mathcal{S},\mathcal{A},\mathcal{T},R,\gamma \right\}$, where $\mathcal{S}$ denotes the state space, $\mathcal{A}$ denotes the action space, $\mathcal{T}$ denotes transition probability, $R$ denotes the reward signal, and $\gamma\in[0,1]$ denotes the discount factor.
In traditional single-agent RL, we seek to find a policy $\pi$ that maximizes the expected return over the states $s\in \mathcal{S}$.
\begin{equation}
    \label{eq:MDP_objective_function}
    \mathcal{J}(\pi)=\mathbb{E}_{\pi}[\sum_{t=0}^{\infty}\gamma^tr_t]
\end{equation}

In MARL, the environment is non-stationary from the perspective of each agent because other agents are learning and changing their policies simultaneously. This makes it challenging for the agents to learn stable policies. One solution is the development of Interactive-POMDP proposed~\cite{gmytrasiewicz2005framework}, which enables agents to utilize more sophisticated techniques to model and predict behaviors of other agents. 
The I-POMDP for the agent {$A^i \in\{A^1,...,A^N\}$} is defined as $\left\{ IS_i,\mathcal{A}, \mathcal{T}, \Omega_i,\mathcal{O}_i,R_i \right\}$, in which
the most notable component is the interactive state $IS_i=\mathcal{S} {\times} M$, where $M$ holds all possible models of other agents.
Apart from this, other components of I-POMDP are similar to a standard POMDP. 
The great I-POMDP contribution lies in incorporating the modeling of other agents into a subject agent's decision optimization, enabling the agents to observe and coordinate with each other.

\subsection{Policy Legible Markov Decision Process}
{Recently, Faria~{\it et al.}\cite{faria2024guess} introduced legibility into traditional MDPs, namely Policy Legible Markov Decision Process~(PoLMDP), denoted as $\left\{ \mathcal{S},\mathcal{A}, \mathcal{T}, {\mathcal{R}},{R},\gamma, \beta \right\}$.
A PoLMDP is defined in the context of an environment with $N$ different objectives, each of which is represented by a different reward function $R_n,n=1,...,N$ and thus with a different MDP $\mathcal{M}_n$.
$\mathcal{R}$ is the legible reward function, denoted as:
\begin{equation}
    \mathcal{R}(s,a)=P(R_n|s,a)
\end{equation}
where $P(R_n|s,a)$ can be reformulated via Bayes' Theorem below.
\begin{equation}
    P(R_n|s,a) \propto P(s,a|R_n)P(R_n)
\end{equation}

This transformation allows us to express the conditional probability of a reward $R_n$ given a state $s$ and action $a$ in terms of the conditional probability of the state and action given the reward, multiplied by the prior probability of the reward.
$P(s,a|R_n)$ can be determined by applying the maximum-entropy principle, with $\beta$ serving as the hyper-parameter. Here, $Q^*_n$ represents the optimal Q-function for the MDP $\mathcal{M}_n$, and is calculated as follows:
\begin{equation}
    P(s,a|R_n)=\frac{\exp({\beta}Q^*_n(s,a))}{\sum_{m=1}^{N}\exp({\beta}Q^*_m(s,a))}
\end{equation}

PoLMDP  has the advantage on the computational simplicity and ease of implementation. However, PoLMDP requires the optimal policy under the current objective during the training, which is challenging in multiagent environments where each agent's policy is constantly updated. 
Moreover, the completeness of PoLMDP is difficult to be guaranteed. 
Hence, expanding upon its framework, we propose a novel reward shaping function that is better suited for multiagent interaction. 
Additionally, we conduct a formal analysis of convergence and completeness on achieving the legibility.}

\section{Active Legibility through Reward Shaping}\label{sec_method}
     
Research of single-agent reinforcement learning has been well explored; however, things get very different when there are more than one agents in one common environment. 
The most challenging is the non-stationarity in a multiagent setting, which is caused by the changing of other agents' policies over time as they learn. 
\textcolor{black}{
Building on the previous exploration of legibility, we first define the Legible Interactive POMDP (LI-POMDP) in Sec. \ref{LIpomdp}, which provides the theoretical foundation for the subsequent work. In Sec. \ref{framework_maal}, we then propose a Multiagent Active Legibility (MAAL) framework. The MAAL framework is an implementation of LI-POMDP and aims to improve the legibility of an individual agent's actions to other agents, thus enabling the subject agent to be more easily modeled by others.
}

\subsection{An Interactive POMDP Framework with Legibility}
\label{LIpomdp}
We consider $N$ collaborative agents {$\{A^1,...,A^N\}$} to achieve the overall goal in a common environment. 
Then we assume that the goal $\mathbf{g}$ can be decomposed into $M\geq N$ sub-goals, e.g. $\mathbf{g} = \{g^k| k=1,\ldots,M\}$, and for each sub-goal $g^k \in \mathcal{G}$,  $g^k$ is assigned to an individual agent $A^i\in \{A^1,...,A^N\}$ . The completion of all sub-goals results in succeeding the tasks. 
To facilitate the legibility development, we formulate the extended POMDP  below from the perspective of individual agents. Similar to the modelling paradigm of interactive POMDPs, the extension is to enhance the generation of legible policies for individual agents by predicting other agents' beliefs over their goals. 
\begin{myDef}
 For a subject agent {$A^i \in\{A^1,...,A^N\}$}, a legible I-POMDP~(LI-POMDP) is defined as: 
\begin{equation}
    LI\text{-}POMDP_i=\left\{
    \mathcal{S}, 
    \mathcal{A}, 
    \mathcal{T}^i,
    \mathcal{O}^i,  
    \Omega^i,       
    {R}^i,
    \mathcal{G},
    \mathcal{B}, 
    \mathcal{I}^i,
    \mathcal{P}^i,
    \mathcal{R}^i
     \right\}
\end{equation}   
\end{myDef}
where :
\begin{itemize}
    \item $\mathcal{S}$ is the state space of the environment.
    \item $\mathcal{A}=\mathcal{A}^{i}\times\mathcal{\tilde{A}}^{-i}$ denotes the joint action space of the multiagent system, in which $\mathcal{A}^{i}$ is the executable action space of agent $A^i$, and $\mathcal{\tilde{A}}^{-i}$ is the observation actions of $A^{-i}$ by agent $i$. 
    \item $\mathcal{T}^i : \mathcal{S} \times \mathcal{A}
    \times \mathcal{S} \rightarrow [0, 1]$ is the transition function.
    \item $\mathcal{O}^i$ is the set of observations the agent $i$ receive from environment.
    \item $\Omega^i:\mathcal{S}\rightarrow\mathcal{O}^i$ is the observation function and controls what $A^i$ can receive in state $s$.
    \item ${R}^i$ is the raw reward signal sent from the environment.
    \item $\mathcal{G}=g^i \times g^{-i}$ denotes the set of goals in a multiagent system, where $g^i$ represents the specific goal of agent $A^i$ and $g^{-i}$ is the combination of goals of agents $A^{-i}$.
    \item $\mathcal{B}=\mathcal{B}^{-i}\times\mathcal{B}^{i}$ is the belief over goals $\mathcal{G}$, in which $\mathcal{B}^{-i} \in \mathbb{R}^{(N-1)|\mathcal{G}|\times 1}$ signifies the belief space of $A^i$'s beliefs about all other agents in the system, and $\hat{\textbf{b}}^{-i} \in \mathcal{B}^{-i}$ detailing $A^i$'s prediction about the goals of $A^{-i}$ {\it (what $A^i$ thinks about the collective goals of all other agents)}. $\mathcal{B}^{i} \in \mathbb{R}^{|\mathcal{G}|\times 1}$ signifies the belief space of $A^i$'s prediction of how $A^{-i}$ views $A^i$~($\hat{\textbf{b}}^{i} \in \mathcal{B}^{i}$). In other words, $\hat{\textbf{b}}^{i}$ can be interpreted as {\it what $A^i$ thinks $A^{-i}$ thinks about $A^i$}.
    \item $\mathcal{I}^i:\mathcal{O}^i\times\Tilde{\mathcal{A}}^{-i}{\longrightarrow}\mathcal{B}^{-i}$ is the function for agent $i$ to infer and predict the goals of agent $A^{-i}$.
    \item $\mathcal{P}^i:\mathcal{O}^i\times\Tilde{\mathcal{A}}^{-i} {\longrightarrow}\mathcal{B}^{i}$ is the function to indicate how others agent $-i$ are predicting agent $i$'s goal.
    \item $\mathcal{R}^i$ is the reward shaping function to enhance the legibility of $A^i$'s policy. It is derived from the original reward $R^i$ by the environment~(to be elaborated in Section~\ref{sec:KLG}). 
    
\end{itemize}

{
We begin with modifying the objective function in Eq.~\ref{eq:MDP_objective_function} and propose a new legible objective function in Eq.~\ref{eq:legible_objective}.
\begin{equation}
    \mathcal{J}(\pi^i)=\mathbb{E}_{\pi^i}[\sum_{t=0}^{\infty}\gamma^tr^i_t-{\beta}D_{KL}(\hat{\textbf{b}}^{i}_{t}||\mathbf{g}^i)]
    \label{eq:legible_objective}
\end{equation}
where $\mathbf{g}^i\in \mathbb{R}^{|\mathcal{G}|\times 1}$ is the true distribution of  agent $A^{i}$ over the goals $\mathcal{G}$, denoted as $\mathbf{g}^i = f_{\text{one-hot}}(g^i)$. $\beta \in \mathbb{R}^+$ is the legibility weight to control the legibility level, $\hat{\textbf{b}}^{i}_{t}$ is the estimation of the predictive distribution of $g^i$ from other agents $A^{-i}$ at the time slice $t$, and $D_{KL}$ is the Kullback-Leibler divergence between $\hat{\textbf{b}}^{i}_{t}$ and $\mathbf{g}^i$.
Eq.~\ref{eq:legible_objective} maximizes the discounted cumulative reward and simultaneously seeks to minimize the margin between the other agent's prediction and the true goal of the subject agent.
}

\subsubsection{Modeling and Predicting Other Agents}
\label{sec:plan_recognition}
Plan recognition involves interpreting the intentions and plans of other agents, as well as anticipating their future actions. 
It largely depends on the analysis of agents' behaviours in context.
By harnessing Bayesian update, we can elevate real-time plan recognition, thereby bolstering the system's overall clarity and transparency.


Initially, we establish a model for agent $A^{i}$ to infer the goals of agent $A^{-i}$ based on observed actions $\Tilde{\mathcal{A}}^{-i}$ and observations $\mathcal{O}^i$. This process is represented as:

\begin{equation}
    \mathcal{I}^i:\mathcal{O}^i\times\Tilde{\mathcal{A}}^{-i}{\longrightarrow}\mathcal{B}^{-i}
\end{equation}
where ~$\hat{\textbf{b}}^{-i}\in \mathcal{B}^{-i}$ computes agent $A^{i}$'s belief over the goals of agents $A^{-i}$ based on its observations $o^i_t \in\mathcal{O}^i$ and observed actions $\Tilde{a}^{j}_{t-1} \in \Tilde{\mathcal{A}}^{j}$. 
This involves calculating agent $A^{i}$'s individual belief about the goals of other $A^j \in A^{-i}$.

At the outset of interaction ($t=0$), agent $A^{i}$ holds a uniform belief about other agents $A^{j}$, denoted as $\hat{\textbf{b}}_{0}^j(g^j) =\frac{1}{|\mathcal{G}|},\forall g^j \in \mathcal{G}$. 
Consequently, a function mapping a policy or trajectory to an agent's goal belief  is to update this prediction. 
For simplicity, we introduce the Bayesian approach~\cite{baker2014modeling}  to update the belief of real-time intention identification:

\begin{equation}
\begin{aligned}
    \hat{\textbf{b}}_{t}^j(g^j)&=\frac{1}{|\mathcal{G}|},\forall g^j \in \mathcal{G}, t=0\\
    \hat{\textbf{b}}_{t}^j(g^j)&=\frac{\hat{\pi}^j(o^i_t,\Tilde{a}^{j}_{t-1}|g^j)}{\sum_{g'\in\mathcal{G}}\hat{\pi}^j(o^i_t,\Tilde{a}^{j}_{t-1}|g')\hat{\textbf{b}}_{t-1}^j(g')}\hat{\textbf{b}}_{t-1}^j(g^j),\forall g^j \in \mathcal{G},t\neq0 
        \label{eq:update_belief}
\end{aligned}
\end{equation}
where $\hat{\pi}^j$ hypothesizes an action distribution over $A^j$. 

Additionally, neural networks, like Recurrent Neural Networks~(RNN) or Long Short-Term Memory~(LSTM), can be incorporated as $\mathcal{I}^i$ due to their ability to process sequential data.
In the end, we concatenate all the individual beliefs with the operator $\odot$ as $\hat{b}^{-i}$ in Eq.~\ref{eq:i}.
\begin{equation}
\label{eq:i}
    \hat{\textbf{b}}^{-i}_{t}= \odot_{j\in[1,N]}^{j\neq i} \hat{\textbf{b}}^{j}_{t}
\end{equation}


After modeling the prediction of other agents' goals, we reverse the process to infer how they might predict our goals. In the LI-POMDP framework, the primary objective is to enhance the agent $i$'s legibility, making its goal $g^i$ more discernible by others. To achieve this, agent $A^{i}$ must be able to reason about how other agents perceive its goals:

\begin{equation}
   \mathcal{P}^i:\mathcal{O}^i\times\Tilde{\mathcal{A}}^{-i} {\longrightarrow}\mathcal{B}^{i}
\end{equation}

{When there are more than two agents in the system, $\mathcal{B}^{i}$ is defined as a weighted sum:
\begin{equation}    
\label{eq:p}
    \hat{\textbf{b}}^i=\frac{1}{N-1}\sum_{j\in [1,N]}^{j \neq i}\omega_{i,j} \hat{\textbf{b}}^{i|j}   
\end{equation}
where $\hat{\textbf{b}}^i \in \mathcal{B}^i $ is the belief of how other agents predict agent $i$'s goal, $\hat{\textbf{b}}^{i|j}$ represents the probability distribution of agent $j$'s prediction of agent $i$'s goals, and $\omega_{i,j}$ can be used to adjust the extent of legibility that agent $i$ expresses to different agents, allowing agent $i$ to prioritize enhancing the legibility towards those agents that are more beneficial to its collaborative efforts.
}

{
The implementation of $\mathcal{P}^i$ varies with the training paradigm. 
In {\it decentralized training} training, agent $i$ utilise an estimator to infer $\hat{\textbf{b}}^i$ based on the actions of other agents $-i$, assessing whether they have understood its intentions through the coordination in their actions. This process, known as {\it recursive reasoning}~\cite{gmytrasiewicz2000rational,de2017negotiating} in game theory, can be exemplified as "I believe that you believe that I believe...". 
In a {\it centralized training} setting, $\mathcal{P}^i$ is more straightforward, as agent $i$ can directly communicate or query~\cite{xuan2001communication,payne2002communicating} the other agents $-i$'s understanding of its intentions (i.e., their estimation of its goals), allowing for more precise and computational-friendly optimization of legibility.}





\subsubsection{Making Legible Decision}

To complete tasks efficiently, $A^i$ chooses an action that is not only to accomplish its own goal $g^i$ but also cooperate with {the $A^{-i}$'s goal} in the decision process. 
Therefore, the policy input for $A^i$ is the concatenation of observation $o_t$, its own goal $g^i$, and the estimation of goals of other agents $\hat{\textbf{b}}^{-i}_{t}$: $\pi^i(o_t \odot \textbf{g}^i \odot \hat{\textbf{b}}^{-i}_{t})\xrightarrow{}a^i_t$.
When the environmental state is transited to $s_{t+1}$ by the joint actions $[a^i_t,a^{-i}_t]$, the reward $r^i_t$  for $A^i$ is received and its policy is updated. 
{At this point, $r^i_t$ will be manipulated to $\Tilde{r}^i_t$ and fed back to $A^i$ to improve the legibility while updating policy.}


We utilize the \textcolor{black}{{\it reward shaping}} technique to enable the agent's behavior to be legible in solving MDP with the new objective function in Eq.~\ref{eq:legible_objective}. 
We define the KL-divergence Gain~(KLG), denoted by $\Delta D_{KL}(\mathcal{A}^i)$, as the difference of $D_{KL}(\hat{\textbf{b}}^{i}||\mathbf{g}^i)$ before and after the action $a^i_{t-1}$ is executed in the state $s_{t-1}$.
\begin{equation}
\label{eq:delta_Dkl}
    \Delta D_{KL}(a^i_{t-1}) = D_{KL}(\hat{\textbf{b}}^{i}_{{t-1}}||\mathbf{g}^i)-D_{KL}(\hat{\textbf{b}}^{i}_{{t}}||\mathbf{g}^i), t \geq 1
\end{equation}

{In Eq.\ref{eq:delta_Dkl}, $\Delta D_{KL}(a^i_{t})$ quantifies the amount of ambiguity reduced once the  action $a^i_{t-1}$ is executed.}
For instance, if the action $a^i_{t-1}$ from $A^i$ is very informative for $A^{-i}$ to distinguish the plan, the $\hat{\textbf{b}}^{i}_{t}$ would converge to the {subject} agent's true goal ${g}^i$ with the execution of $a^i_t$:$\hat{\textbf{b}}^i({g}^i)\xrightarrow{a^i_{t-1}}1$
, and $D_{KL}(\hat{\textbf{b}}^{i}_{t}||\mathbf{g}^i)$ would therefore reduce with it, and vice versa. 
Then, we add the legibility term $\Delta D_{KL}(a^i_{t-1})$ to the original reward signal $r^i_{t-1}$ with the parameter $\beta$ to balance the scale.

\begin{equation}
    \Tilde{r}^{i}_{t-1}=r^{i}_{t-1} + {\beta}\Delta D_{KL}(a^i_{t-1})
    \label{reward_shaping}
\end{equation}
 
With the legibility incorporation, $\Tilde{r}^{i}_{t-1}$  is used in the policy update of $A^i$ through reinforcement learning. 
To summarize it, the legibility works in the following way: if $a^i_t$ is helpful for $A^j\in A^{-i}$ to recognize the true goal of $A^i$, the KL-divergence between $\hat{\textbf{b}}^i$ and $\mathbf{g}^i$ reduces, leading to a positive $\Delta D_{KL}$, and eventually encouraging $A^i$ to be more likely to choose $a^i_t$ by increasing the reward.

In the end, for each agent, the problem can be simplified to an MDP and solved using single-agent reinforcement learning methods (with the environmental uncertainty already encapsulated in $(o^i \odot \textbf{g}^i \odot \textbf{b}^{-i})$). The policy update is performed using the tuple \([(o^i_{t-1}\odot \textbf{g}^i \odot\textbf{b}^{-i}_{t-1}), a^i_{t-1}, (o^i_{t}\odot \textbf{g}^i \odot\textbf{b}^{-i}_{t}), \tilde{r}^i_{t-1}]\).

 
  
 \subsection{The MAAL Framework}
 \label{framework_maal}
In this section, we delve into the MAAL framework, elucidating its workings through a scenario as shown in Fig.~\ref{fig:legibility_rewardshaping}.
The scenario involves $N$ agents, categorized as $A^i$ and $A^{-i}$. Their shared objective, denoted as $\mathcal{G}$, is divisible into the sub-goals: $\mathcal{G} = g^i \cup g^{-i}$. A Success is achieved when agent $A^i$ reaches $g^i \in \mathcal{G}$ and agents $A^{-i}$ attain $g^{-i} \in \mathcal{G}$, with mutual awareness of the accomplishment. The dashed lines colored by orange in Fig.~\ref{fig:legibility_rewardshaping} represent the operational framework from the  perspective of agent $A^i$, equally applicable to $A^{-i}$.
 \begin{figure}[h]
    \centering
    \includegraphics[width=1.0\linewidth]{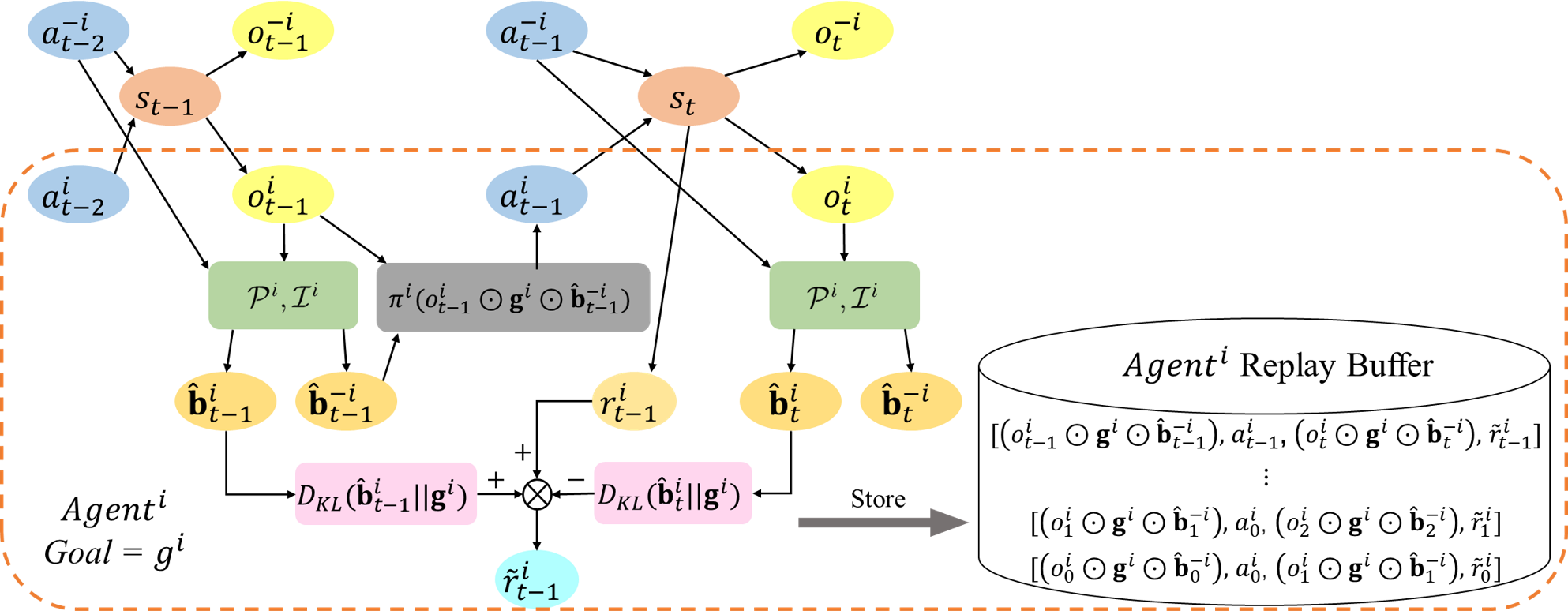}
    \caption{The new framework of Multiagent Active Legibility~(MAAL) is presented from the single-agent perspective~($Agent^i$).}
    \label{fig:legibility_rewardshaping}
\end{figure}

We elaborate the details by following the execution sequence in MAAL.
At the time step $t$ and state $s_t$, agent $A^i$ first identifies $A^{-i}$'s goal via plan recognition~(signified by the green block labeled $\mathcal{I}^{i}$). The sub-model $\mathcal{I}^{i}$ computes the probability distribution $\hat{\textbf{b}}^{-i}_{t}$~(indicated by the gold circle) over potential targets of $A^{-i}$ based on observed states and actions. 
Following this recognition, $A^i$ assesses $A^{-i}$'s beliefs through another plan recognition process (denoted as $\mathcal{P}^{i}$). The sub-model $\mathcal{P}^{i}$ estimates the belief probability distribution $\hat{\textbf{b}}^{i}_{t}$ (the gold circle) concerning $A^{i}$'s targets from $A^{-i}$'s perspective. Then, agent uses the KL-divergence Gain~(KLG, the pink block) to evaluate the legibility of action $a^i_{t-1}$. With the legibility incorporation, the agent receives a shaped reward $\Tilde{r}^{i}_{t}$  used in the policy update of $A^i$.  

Complementing Fig.~\ref{fig:legibility_rewardshaping}, we present the MAAL framework for $N$ agents in Alg.~\ref{alg:maal}.
Lines 2-3 initialize the system at the start of each episode. Agents establish their goals and reset their beliefs about other agents' goals, assuming a uniform distribution over the each potential goals. At every time step $t$, the agent assesses other agents' goal beliefs using Eq.~\ref{eq:i} and~Eq.\ref{eq:update_belief} or alternative mapping methods~(Lines 6-9). It then estimates other agents' predictions of its own goals using Eq.~\ref{eq:p} (Line 10). The agent subsequently generates an action based on the policy $\pi^i(o^i_{t-1}\odot \textbf{g}^i \odot\textbf{b}^{-i}_{t-1})$~(Line 11). Line 12 outlines the environmental state transition with their joint action and feedback of reward signals.
Subsequently, the agent adjusts the original rewards. 
Specifically, $A^i$ calculates $\Delta D_{KL}$ (Line 20) and derives the legibility reward $\Tilde{r}^i_{t-1}$~(Line 21). 
Finally, $A^i$ stores the tuple \([(o^i_{t-1}\odot \textbf{g}^i \odot\textbf{b}^{-i}_{t-1}), a^i_{t-1}, (o^i_{t}\odot \textbf{g}^i \odot\textbf{b}^{-i}_{t}), \tilde{r}^i_{t-1}]\) in the replay buffer for single-agent reinforcement learning~(Line 22).


In conclusion, the MAAL framework facilitates mutual understanding and cooperation among multiple agents by enabling them to recognize and adapt to others' goals. Through a combination of plan recognition and belief estimation, agents can effectively collaborate towards shared objectives while enhancing the legibility of their actions for improving their performance.

\begin{algorithm}[htbp]
\SetAlgoLined
\caption{Multiagent Active Legibility}
\label{alg:maal}
\normalem
\KwIn{
$N$ agents: $\{A^1,...,A^N\}$, 
$N$ LI-POMDP: $\{\mathcal{M}^1,...,\mathcal{M}^N\}$,
Legibility weight: $\beta\in\mathbb{R}^+$, 
Maximum train episodes: $M\in\mathbb{N}^+$, 
Maximum episode steps: $L\in\mathbb{N}^+$}
\KwResult{N agent policies: $\pi^{1},...,\pi^{N}$}
\While{$episodes \leq M$ }{   
    Reset the environment to $s_0$\;
    Each agent is assigned with a goal: $g^i,i=1,...,N$\;
    \While{$t \leq L$}
    {
        \For{Agent {$A^i \in\{A^1,...,A^N\}$}}
        {
             
             \For{Agent {$A^j \in\{A^1,...,A^N\}\setminus \{A^i\}$}}
            {
             Estimate the belief $\hat{\textbf{b}}^{j}_{t}$ over the goals of agent $A^{j}$ 
             based on the observation $ o^{i}_{t},\Tilde{a}^{j}_{t-1}$ via~Eq.\ref{eq:update_belief}
              }
            Update the overall belief $\hat{\textbf{b}}^{-i}_{t}$
            via~Eq.\ref{eq:i} \\
            Update the belief $\hat{\textbf{b}}^{i}_{t}$ 
            via~Eq.\ref{eq:p}  
   
            Choose the action via $\varepsilon$-greedy\cite{Sutton} algorithm\\
            $ a^{i}_{t} \leftarrow \pi^i(o^{i}_t \odot \textbf{g}^i\odot \hat{\textbf{b}}^{-i}_{t})$\;
        }
        Environment takes the joint action $[a^0_t,...a^N_t]$, transit to the next state, and feedback the rewards $[r^1_t,...r^N_t]$\;
       
            \For{Agent {$A^i \in\{A^1,...,A^N\}$}}
            {
            Calculate the KL-divergence Reduction:\\
             \If{$t = 0$}{
                 $\Delta D_{KL}(a^i_{t-1}) \leftarrow 0$
                }
            \Else{
                $\Delta D_{KL}(a^i_{t-1}) \leftarrow  D_{KL}(\hat{\textbf{b}}^{i}_{t-1}||\textbf{g}^i)-D_{KL}(\hat{\textbf{b}}^{i}_{t}||\textbf{g}^i)$ 
            }
            Calculate the legibility reward: $\Tilde{r}^i_{t-1} \leftarrow  r^i_{t-1}+{\beta}\Delta D_{KL}(a^i_{t-1})$
            
            Update the policy $\pi^i$ with tuple \([(o^i_{t-1}\odot \textbf{g}^i \odot\textbf{b}^{-i}_{t-1}), a^i_{t-1}, (o^i_{t}\odot \textbf{g}^i \odot\textbf{b}^{-i}_{t}), \tilde{r}^i_{t-1}]\)\; 
            
        }
        $t \leftarrow  t + 1$\;
    }
    $episodes \leftarrow  episodes + 1$\;}
    
\end{algorithm}

\section{Convergence Analysis}
\label{canalysis}
In this section, we first define the completeness in Definition \ref{def:complete}, and analysis the MAAL reward shaping from the perspective of a single agent. Specifically, we examine if an agent can reach a sub-goal state under the original reward function, and if, using MAAL for the reward shaping, it satisfies Proposition~\ref{rs}, ensuring that the agent with enhanced legibility can also reach the target state. We first define the completeness of a policy $\pi$ below.

\begin{definition}
\label{def:complete}
Let $s^0$ represent the initial state, $s^*$ the goal (absorbing) state, and $P_{\pi}^{t}(s^i,s^j)$ denote the probability of an agent reaching $s^j$ after $t$ steps transitions from $s^i$. Then a policy $\pi$ is complete, if for every initial state $s^0$, there exist a finite time $t$ such that $P_{\pi}^{t}(s^0,s^*)=1$.
\end{definition}
\begin{proposition}
Let any agent $A^i$ with a target sub-goal $g^i$ be given.
The necessary condition for applying the shaping reward by MAAL algorithm to agent $A^i$ without compromising the completeness of the policy as defined in Definition \ref{def:complete} is: 
at any two time $t_1$ and $t_2$, when the agent $A^i$ is in the same state $s$, the observing agent's estimation of agent $A^i$'s goal must be consistent, i.e. $\hat{\textbf{b}}_{t_1}^i=\hat{\textbf{b}}_{t_2}^i$
\label{rs}
\end{proposition}
\begin{proof}
We refer to the convergence analysis of reward shaping~\cite{ng1999policy}, which reveals that one way an agent’s policy is incompleteness when agents repeatedly visit non-goal states in pursuit of shaping rewards, thereby hindering task completion.
To further clarify it, we introduce a straightforward example to illustrate the potential issues arising from the improper reward shaping in multi-agent systems.
\begin{figure}[h!]
    \centering
    \subfigure[Unshaping Reward]{\includegraphics[width=0.45\linewidth]{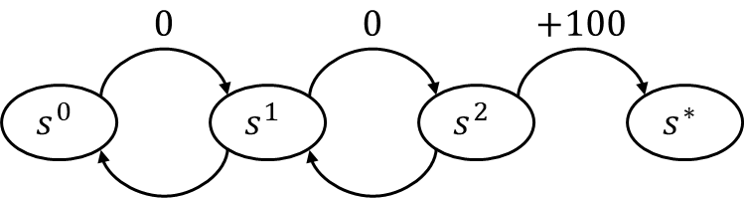}
    \label{fig:unshaping}}
    \subfigure[Shaping Reward]{\includegraphics[width=0.45\linewidth]{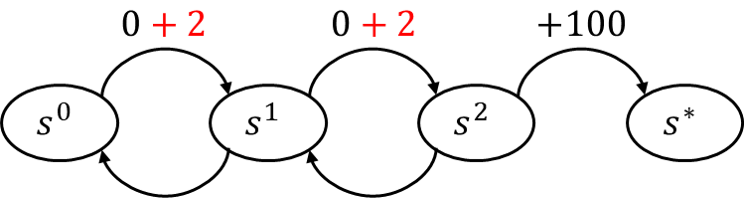}
    \label{fig:shaping}}
    \caption{An example of Reward Shaping in MDP}
    \label{fig:reward shaping example}
\end{figure}
Fig.\ref{fig:reward shaping example} illustrates a scenario where reward shaping is used to improve the agent's exploration efficiency. Although the purpose of using reward shaping in this context differs from this paper (where MAAL uses reward shaping to directly alter the final policy), it is still highly valuable for analyzing completeness.

Fig.\ref{fig:unshaping} illustrates a MDP consisting of four states: \{\( s^0 \), \( s^1 \), \( s^2 \), and \( s^* \)\}, in which the \( s^0 \) is the initial state, and the agent's objective is to reach the target state \( s^* \). The reward function for each transition is indicated above the arrows. As shown, the agent receives the reward of +100 only upon making the final transition. Hence, the agent may require extensive exploration and trials to discover the state transition that completes the task.
To enhance the agent's exploration efficiency, Fig.~\ref{fig:shaping} introduces an additional reward of $+2$~(red numbers over the arrow) for the transitions \( s^0 \rightarrow s^1 \) and \( s^1 \rightarrow s^2 \) on top of the original reward function. This adjustment is intended to guide the agent towards states \( s^1 \) and \( s^2 \), thereby improving its exploration efficiency. 
However, this reward function introduces a new problem: if the MDP has an infinite duration, the agent can continuously cycle between \( s^0 \), \( s^1 \), and \( s^2 \) to accumulate rewards that exceed the reward for reaching the target state \( s^* \) via \( s^0 \rightarrow s^1 \rightarrow s^2 \rightarrow s^* \). In such a scenario, the agent accumulates significant rewards but fails to achieve the initial goal, thus losing the completeness of the algorithm.

Back to the transition loop with the legibility reward shaping in MAAL. Assume agent \( A^i \) has the goal \( g^i \). Starting from the state \( s^0 \) at time \( t = 0 \), the agent reaches the state \( s^n \) after \( n \) time steps and then returns to \( s^0 \) at time \( t = n+1 \), forming a loop as: $s^0_{t=0} \rightarrow ... \rightarrow s^n_{t=n} \rightarrow s^0_{t=n+1}$. Then the discounted return received by agent \( A^i \) is~(assuming the legibility weight \( \beta = 1 \) for the simplicity):
\begin{equation}
\begin{aligned}
    R&=\sum_{k=0}^{n+1}
    \gamma^k r^i_k+D_{KL} (\hat{\textbf{b}}_{t=0}^i || g^i) - D_{KL}(\hat{\textbf{b}}_{t=1}^i || g^i)+\gamma (D_{KL}(\hat{\textbf{b}}_{t=1}^i || g^i) - \\
    & D_{KL} (\hat{\textbf{b}}_{t=2}^i || g^i))+\gamma^2(D_{KL}(\hat{\textbf{b}}_{t=2}^i || g^i) -D_{KL}(\hat{\textbf{b}}_{t=3}^i || g^i)) + ... + \\
    & \gamma^n (D_{KL}(\hat{\textbf{b}}_{t=n}^i || g^i) -D_{KL}(\hat{\textbf{b}}_{t=n+1}^i || g^i))
\end{aligned}
\end{equation}
Compared to the original reward signal $\sum_{k=0}^{n+1}{\gamma^k r^i_k}$, MAAL provides an additional reward signal $R_{KLG}$ to agent \( A^i \):
\begin{equation}
    \begin{aligned}
        R_{KLG} &= R - \sum_{k=0}^{n+1}{\gamma^k r^i_k} \\
        &=D_{KL} (\hat{\textbf{b}}_{t=0}^i || g^i) - D_{KL}(\hat{\textbf{b}}_{t=1}^i || g^i)+\gamma (D_{KL}(\hat{\textbf{b}}_{t=1}^i || g^i) - \\
    & D_{KL} (\hat{\textbf{b}}_{t=2}^i || g^i))+\gamma^2(D_{KL}(\hat{\textbf{b}}_{t=2}^i || g^i) -D_{KL}(\hat{\textbf{b}}_{t=3}^i || g^i)) + ... + \\
    & \gamma^n (D_{KL}(\hat{\textbf{b}}_{t=n}^i || g^i) -D_{KL}(\hat{\textbf{b}}_{t=n+1}^i || g^i))
    \end{aligned}
\end{equation}
Since the discount $\gamma \in (0,1]$, we have:
\begin{equation}
    \begin{aligned}
        R_{KLG} &\leq D_{KL} (\hat{\textbf{b}}_{t=0}^i || g^i) \underbrace{- D_{KL}(\hat{\textbf{b}}_{t=1}^i || g^i)+ D_{KL}(\hat{\textbf{b}}_{t=1}^i || g^i)}_{= 0} \\ 
    & \underbrace{ - D_{KL} (\hat{\textbf{b}}_{t=2}^i || g^i)+D_{KL}(\hat{\textbf{b}}_{t=2}^i || g^i)}_{=0} \underbrace{-D_{KL}(\hat{\textbf{b}}_{t=3}^i || g^i) +D_{KL}(\hat{\textbf{b}}_{t=3}^i || g^i)}_{=0} +... + \\
    & \underbrace{-D_{KL}(\hat{\textbf{b}}_{t=n}^i || g^i)+D_{KL}(\hat{\textbf{b}}_{t=n}^i || g^i)}_{=0} -D_{KL}(\hat{\textbf{b}}_{t=n+1}^i || g^i)
    \end{aligned}
\end{equation}
Therefore, we can derive an upper bound for \( R_{KLG} \) below.
\begin{equation}
     R_{KLG} \leq \underbrace{D_{KL} (\hat{\textbf{b}}_{t=0}^i || g^i) - D_{KL} (\hat{\textbf{b}}_{t=n+1}^i || g^i)}_{=0}
\end{equation}
 where \( \hat{\textbf{b}}_{t=0}^i \) represents the prediction distribution of other agents for \( A^i \) at time \( t=0 \) when it is in the state \( s^0 \). 
And \( \hat{\textbf{b}}_{t=n+1}^i \)  is the prediction distribution of other agents for \( A^i \) at time \( t=n+1 \) when it is in state \( s^0 \) and has the trajectory: $s^0_{t=0} \rightarrow ... \rightarrow s^n_{t=n} \rightarrow s^0_{t=n+1}$.

In summary, to ensure that the agent $ A^i $ receives a non-positive reward $ R_{KLG} \leq 0 $, 
the observer's model must have the same goal estimation of $A^i$ in one state: $D_{KL} (\hat{\textbf{b}}_{t=0}^i || g^i) = D_{KL} (\hat{\textbf{b}}_{t=n+1}^i || g^i)$
\end{proof}

\section{Experimental Results}
\label{sec_exp}

This section evaluates the MAAL performance on specific problems~(emphasizing goal identification between agents) through comparative experiments and ablation studies. 
In Section~\ref{sec_lead-follow-maze}, we first introduce the Lead-Follow Maze~(LFM), a discrete maze scenario with two agents and four exits. 
To diversify the experimental environments, we propose the {\it simple navigation} scenario in Section~\ref{sec_Particle_Simple_Navigation}, which expands upon the simple spread task from the {\it particle} environment, incorporating more agents and landmarks. The simple navigation, with its continuous state-action space, allows for a more comprehensive evaluation of MAAL in complex settings, testing its ability to enhance action legibility and improve multi-agent system performance.
\subsection{Lead-Follow Maze}
\label{sec_lead-follow-maze}
\begin{figure}[h]
    \centering
\includegraphics[width=0.6\linewidth]{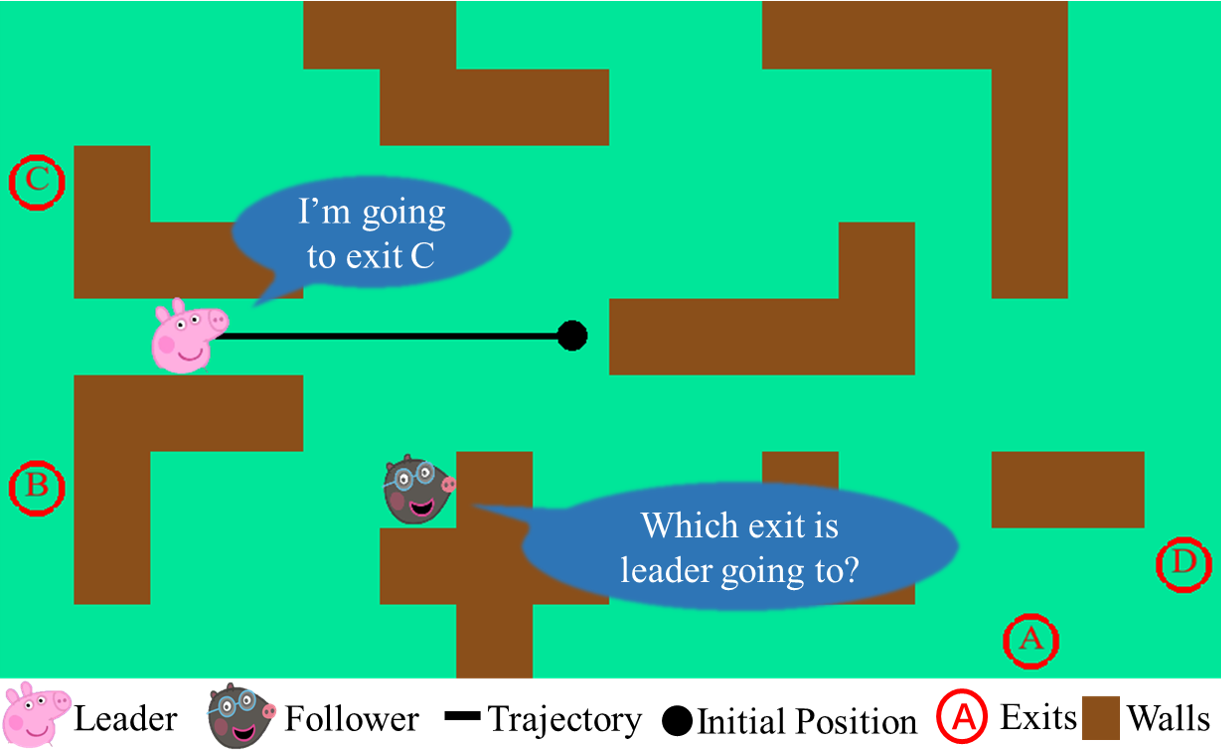}
    \caption{The Lead-Follow Maze domain: 
    In a $10\times 16$ maze, there are two independent agents, the leader and the follower. Both of them can observe each other's actions and trajectories. 
    There are four exits distributed around the maze, as well as several walls. 
    At the beginning of each game, the leader is assigned a target exit, which is unknown to the follower. 
    The objective of the game is for both agents to reach their respective target exits simultaneously, thereby ending the game and achieving victory. Therefore, the follower needs to observe the leader's trajectory, infer the true target exit, and hurry to it. Meanwhile, the leader can also enhance the legibility of its actions and strategies to help the follower identify its target more quickly and accurately. }
\label{fig: domain}
\end{figure}

\subsubsection{Experimental Settings}
Since both the state and action spaces are discrete, the leader and follower employ Q-Learning to learn their policies.
In the LFM environment, the state space for both agents is defined as $\mathcal{S}$, which represents the coordinates of the leader and follower. For example, as shown in Fig.~\ref{fig: domain}, $s=((2,4),(5,2))$, where $(2,4)$ indicates the coordinate of the leader, and $(5,2)$ represents the coordinate of the follower.
For the Leader agent, the policy is fed with $[\mathcal{S} \odot g^*]$, where $s$ is the state and $g^*$ represents the target exit: $a^L \leftarrow argmax_{a' \in \mathcal{A}}Q^L([\mathcal{S} \odot g^*], a') $. 
As to the follower agent, its input is defined as $[\mathcal{S}\odot \hat{g}]$, where $\hat{g}$ is the follower's estimation of the true target, obtained from the Bayesian learning as Eq.~\ref{eq:update_belief}. At the end of the episode, the follower utilizes the observation of leader's trajectory and the true target to update its plan recognition through parameter learning. 

The action space for both agents is $\mathcal{A} = \left\{up, down, left, right, stay\right\}$, where each action moves one grid, and the state transition is deterministic. 
When the agents complete the task, they both receive a reward of $+1$. For each grid it moves, the agent receives the motion cost of $-0.1$.

\subsubsection{Comparative Experiments}
In comparative experiments, we select a series of the MARL algorithms compared to the MAAL approach in the LFM domain.
\begin{itemize}
    \item Independent Q-Learning~(IQL)~\cite{tan1993multi} is an extension of the Q-Learning algorithm for multiagent settings. Multiple agents are trained independently, each with its own policy. Agents interact with the environment and learn to maximize their individual expected rewards. 
    \item Value-Decomposition Networks~(VDN)~\cite{sunehag2017value} assumes that the global value can be represented as the sum of the individual values of each agent and aims to capture the interdependencies between agents by decomposing the global value function.
    \item QMIX~\cite{rashid2020weighted} employs deep neural networks to learn the value functions and the mixing network, allowing for more expressive representations and approximation of complex value functions compared to VDN.
    \item Multiagent Variational Exploration~(MAVEN)~\cite{mahajan2019maven} is an improved algorithm of the QMIX that overcomes the low exploration efficiency due to monotonicity constraints.
    \item  Mutil-Agent Deep Deterministic Policy Gradient~(MADDPG) \cite{lowe2017multi} is an extension of the Deep Deterministic Policy Gradient (DDPG), where each agent maintains its local actor network to make decisions, and learning policy from a centralized critic network taking the joint action and observation of all the agents as inputs.
\end{itemize}

We have run each algorithm for 5 to 6 times and plot both the mean~(with a curve) and standard deviation~(with a shallow shadow). 

\begin{figure}[h]
\centering
\includegraphics[width=0.6\linewidth]{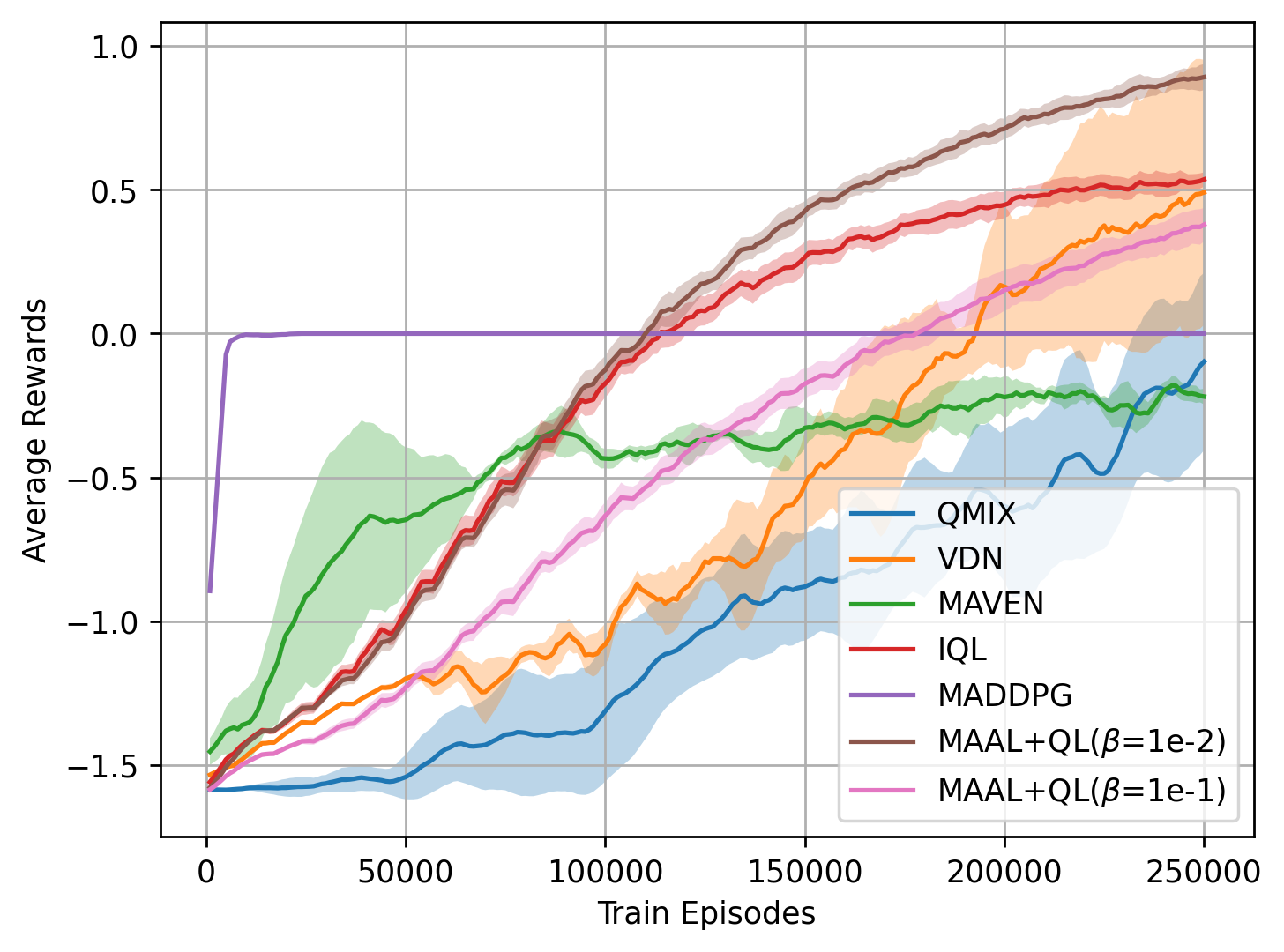}
\caption{Episode Reward in LFM}
\label{fig:sum_reward}
\end{figure}

Figure~\ref{fig:sum_reward} presents the average rewards for the various algorithms, excluding shaped rewards. The results show that, with the exception of MADDPG, all the algorithms exhibit varying degrees of performance improvements in the post-training. The MADDPG algorithm~(purple curve) fails to converge, likely due to its inadequate adaptation to discrete state and action spaces.
Interestingly, despite being one of the simplest algorithms, the IQL algorithm (red curve) performs remarkably well in this environment. Similarly, the VDN algorithm (orange curve) demonstrates good performance; however, it suffers from high variance and significant uncertainty. Although QMIX (blue curve) is designed as an improved version of VDN, it performs poorly in this environment, significantly lagging behind VDN. This discrepancy may arise from the difficulty of training the non-linear combinations in QMIX, which do not adapt well to the environment. A similar pattern is observed with the MAVEN algorithm (green curve). While it shows rapid improvement in the initial stage, its final performance only matches that of QMIX.

On the other hand, the MAAL+QL algorithm achieves the best performance in this environment. When the legibility weight $\beta$ is set to 0.01, MAAL+QL (brown curve) outperforms all comparative algorithms. This outcome suggests that enhancing the legibility of an agent's policy improves the speed and accuracy of intention recognition between agents, thereby boosting collaborative performance. However, when $\beta$ is increased to 0.1, the performance of MAAL+QL declines significantly. The higher reward shaping weight likely causes agents to overemphasize intention expression (i.e., improving legibility) at the expense of task completion efficiency.

\subsubsection{Legibility's Impact Experiment}
\label{sec:exp2}

We conducted comparative experiments with different legibility weight $\beta$ under the same initial conditions. This experiment focused on two metrics: Prediction Correctness Ratio~(PCR) and Prediction Time Ratio~(PTR). PCR refers to the accuracy of follower's predictions of leader' goals. For example, if within training episodes 1000 to 2000, the follower correctly predicted the leader's goal in 415 out of 1000 episodes, then the prediction correctness ratio at Episode 2000 is $415/1000 = 0.415$. PTR refers to the ratio between the number of steps required for an agent to correctly predict other agents' goals from the beginning of the episode to the length of the episode. A smaller PTR means that the agent can accurately predict other agents' goals earlier, suggesting higher legibility. For instance, if an observer correctly predicts the goal at Step 15 and maintains it to the end of the episode, with the total episode length being 50 steps, then the prediction time ratio for that episode is $15/50 = 0.3$.

\begin{figure}[h]
\centering
\subfigure[PCR with different $\beta$ values.]{
\includegraphics[width=0.45\linewidth]{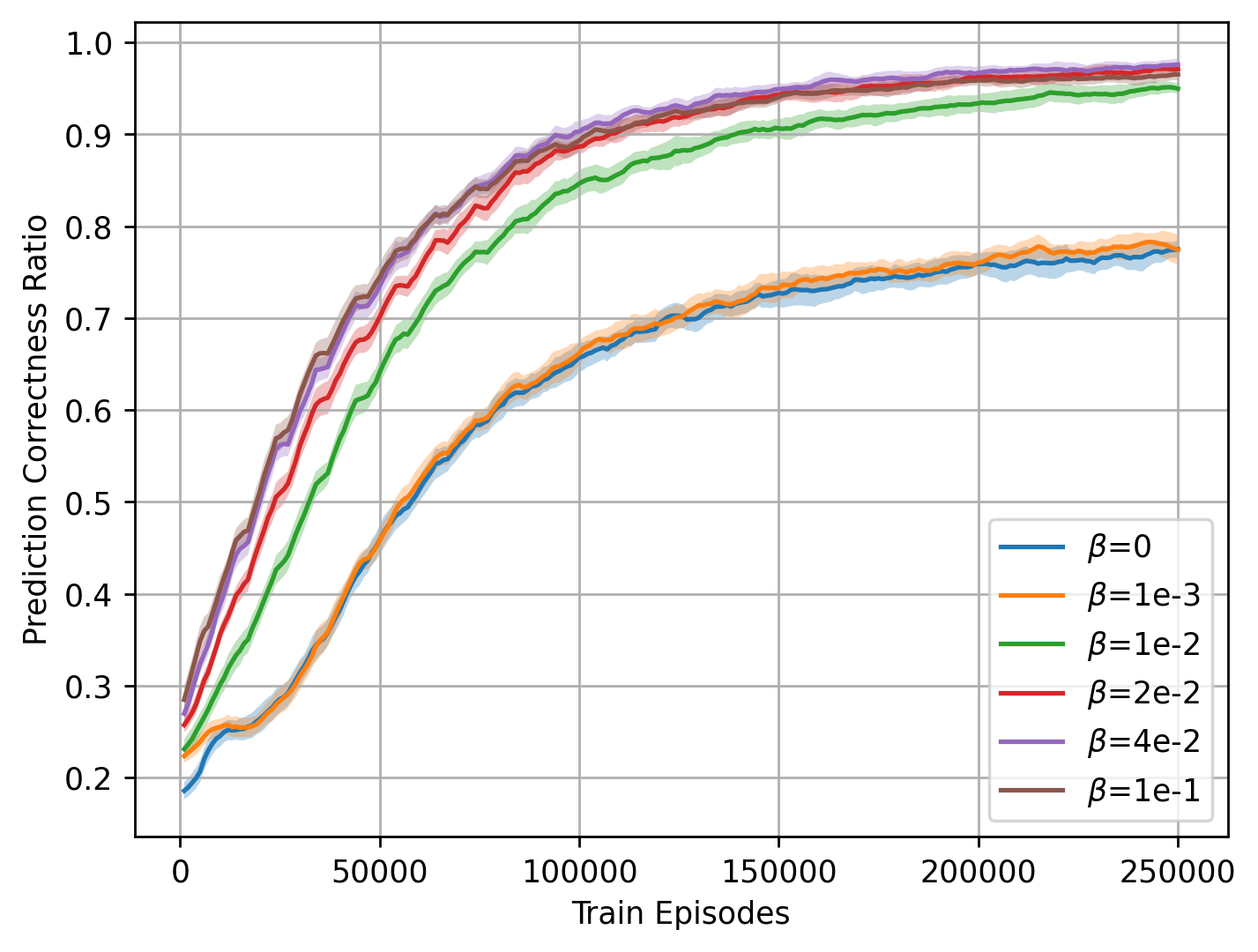}
\label{fig: Prediction Correctness Ratio}
}
\subfigure[PRT with different $\beta$ values.]{\includegraphics[width=0.45\linewidth]{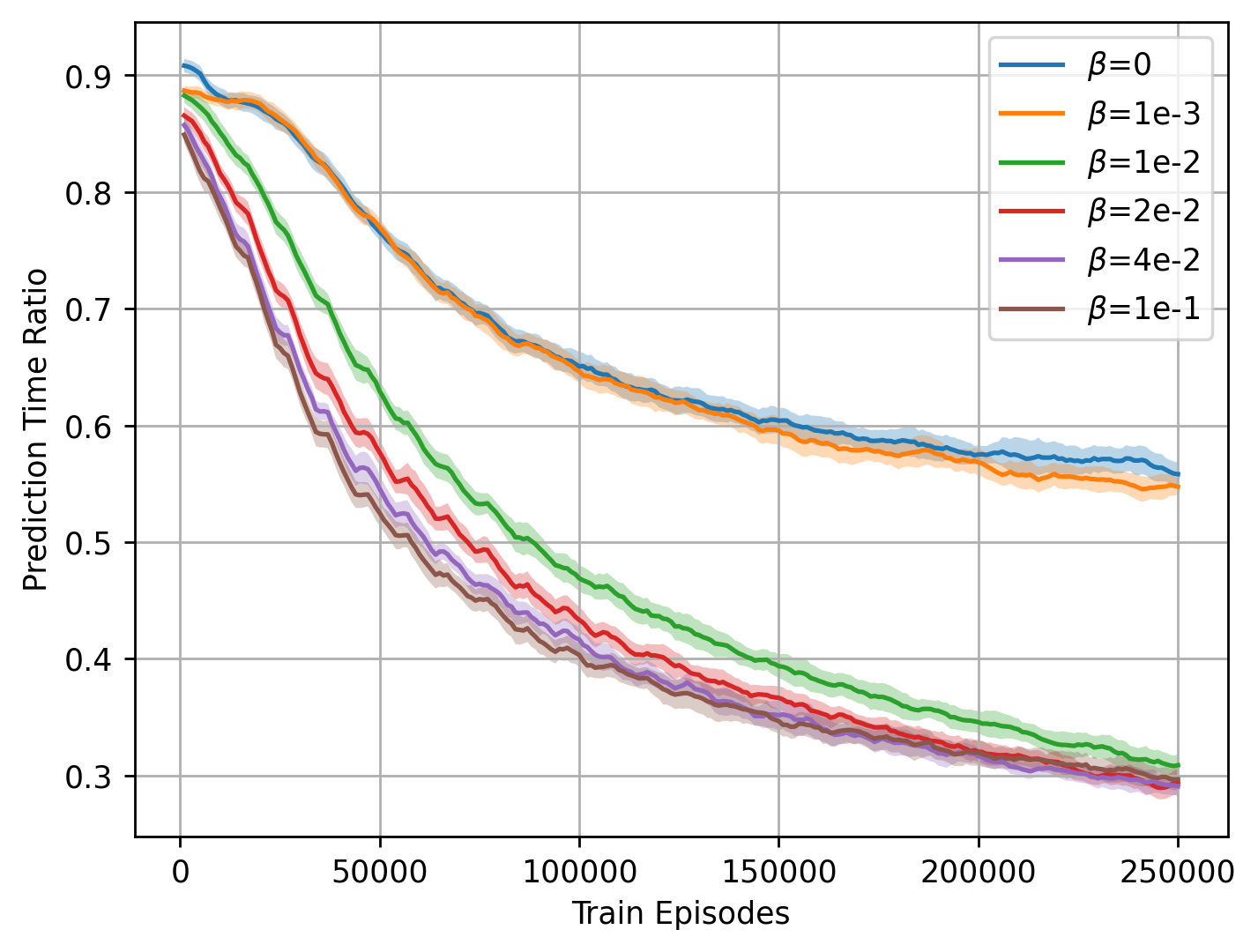}
\label{fig: Prediction Time Ratio}}

\end{figure}

We show the PCR in Fig.~\ref{fig: Prediction Correctness Ratio}~and PTR in Fig.~\ref{fig: Prediction Time Ratio} with different $\beta$ values. 
It is noticed that the monotonic increase in PCR and decrease in PTR happen as $\beta$ grows, which strongly indicates the improved legibility of the {subject} agent's behaviors. 
When the legibility is not applied, the follower's PCR is less than $70\%$, and requires nearly half of the journey before identifying the true goal of subject agent. 
Subsequently, as $KLG$ exerts more influence on the reward signal, PCR climbs to nearly $100\%$, and the PTR is almost halved. 
\subsubsection{MAAL Beyond Q-Learning}
\label{sec:exp3}
As we mentioned before, MAAL stands upon the standard MDP and thus can be integrated into any single-agent reinforcement learning algorithm by solving MARL problems. 
We empirically study the legibility application in two different reinforcement learning algorithms.
In this experiment, we incorporate MAAL with the State Action Reward State Action~(SARSA) and Deep Q-Network~(DQN) methods respectively, denoted as MAAL+SARSA and MAAL+DQN. 
To evaluate whether the legibility is helpful for multiagent learning, we set the parameter $\beta=0$ in one of the experiments, i.e. removing the legibility weight in the reward function.

\begin{figure}[h]
    \centering
    \subfigure[Success Rate of MAAL+SARSA]{\includegraphics[width=0.31\linewidth]{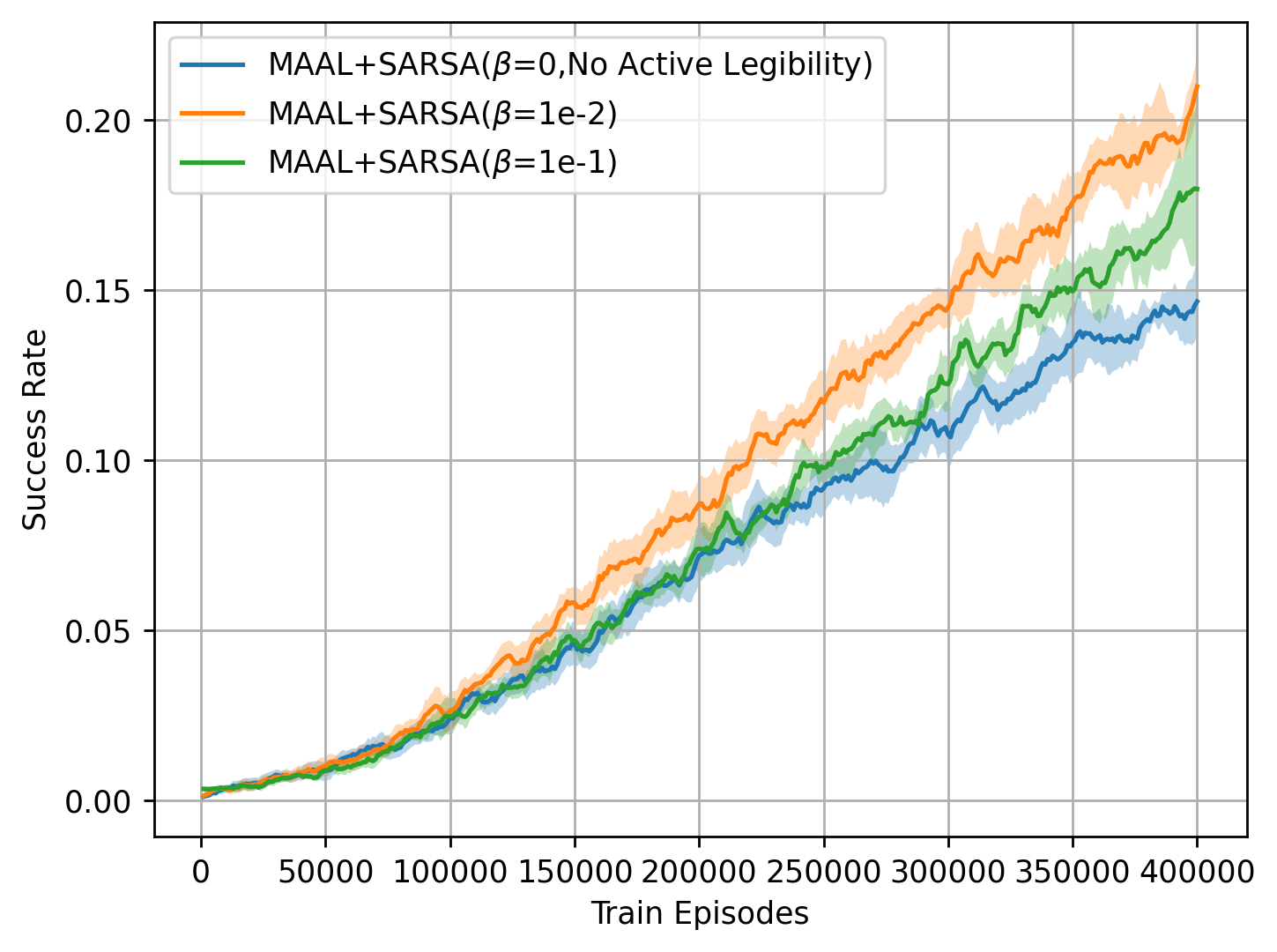}
    \label{fig:SARSA success_rate}}
    \subfigure[PCR of MAAL+SARSA]{\includegraphics[width=0.31\linewidth]{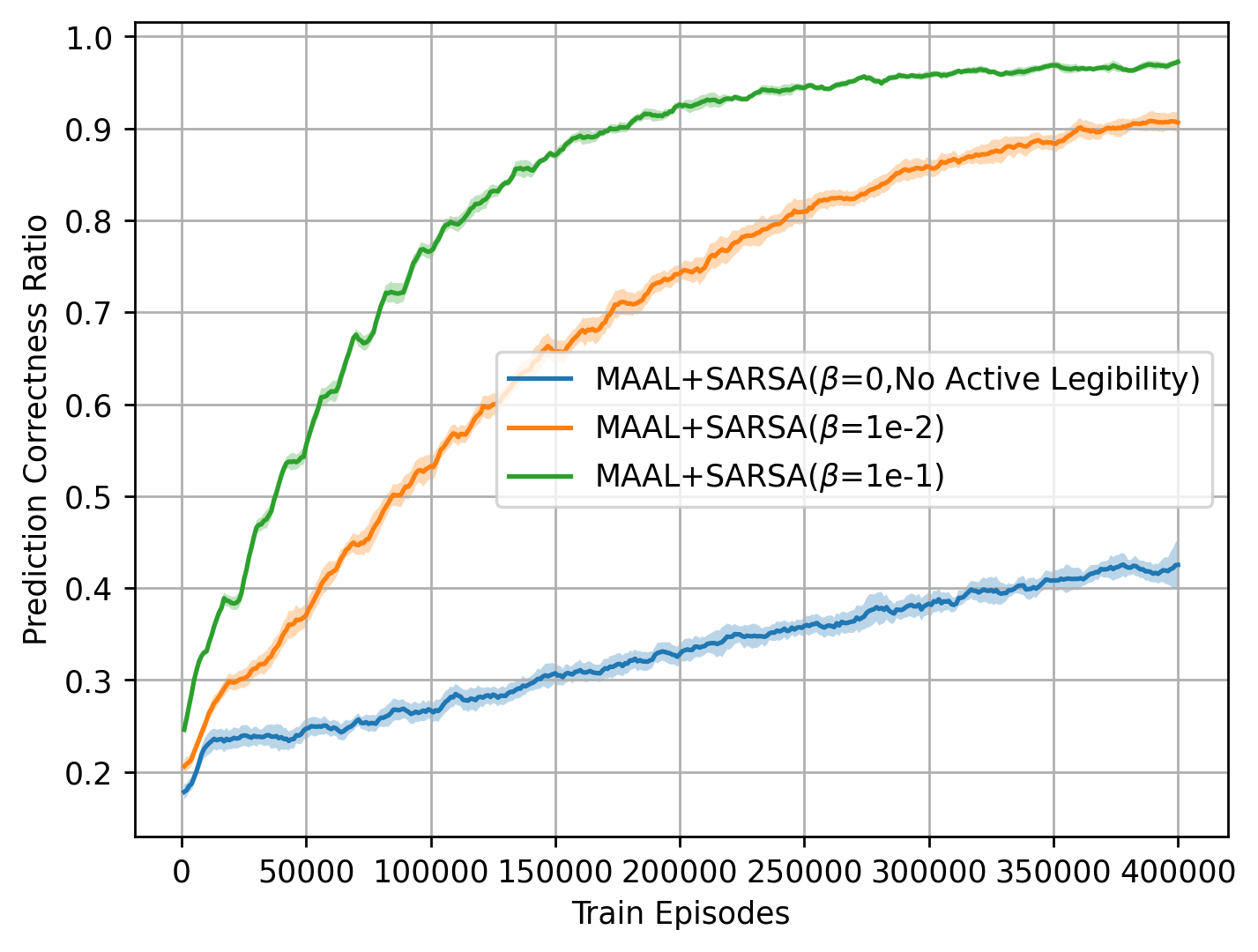}
    \label{fig:SARSA Prediction Correctness Ratio}}
    \subfigure[PTR of MAAL+SARSA]{\includegraphics[width=0.31\linewidth]{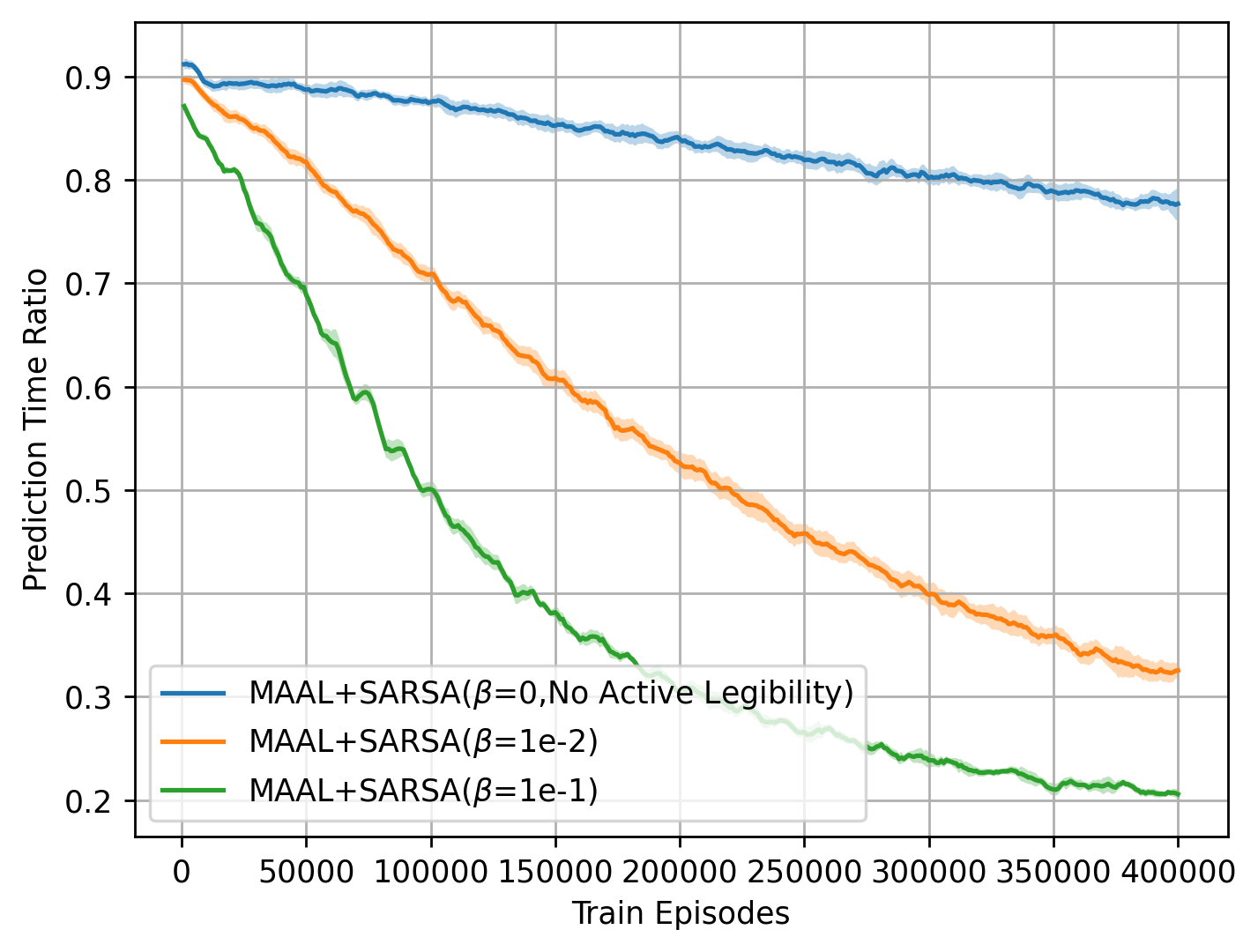}
    \label{fig:SARSA Prediction Time Ratio}}\\
    \caption{Performance of the MAAL with SARSA}
    \label{fig:Combine MAAL with SARSA}
\end{figure}
{Fig.~\ref{fig:SARSA success_rate} has shown, in MAAL+SARSA, the success rate with certain legibility~(orange curve) reaches only $20\%$, but is still superior to those without legibility~(blue curve) in $15\%$. Improving the weight of legibility~(green curve) will cause a decrease in the success rate. The same phenomenon also occurs on MAAL+DQN in Fig.~\ref{fig:DQN success_rate}. In MAAL+DQN, the usage of legibility raises the success rate by $10\%$. However, after increasing the legibility weight to $10^{-1}$, the success rate of DQN sharply decreases, as DQN is more sensitive to rewards compared to SARSA and Q-Learning.
\begin{figure}[h]
    \subfigure[Success Rate of MAAL+DQN]{\includegraphics[width=0.31\linewidth]{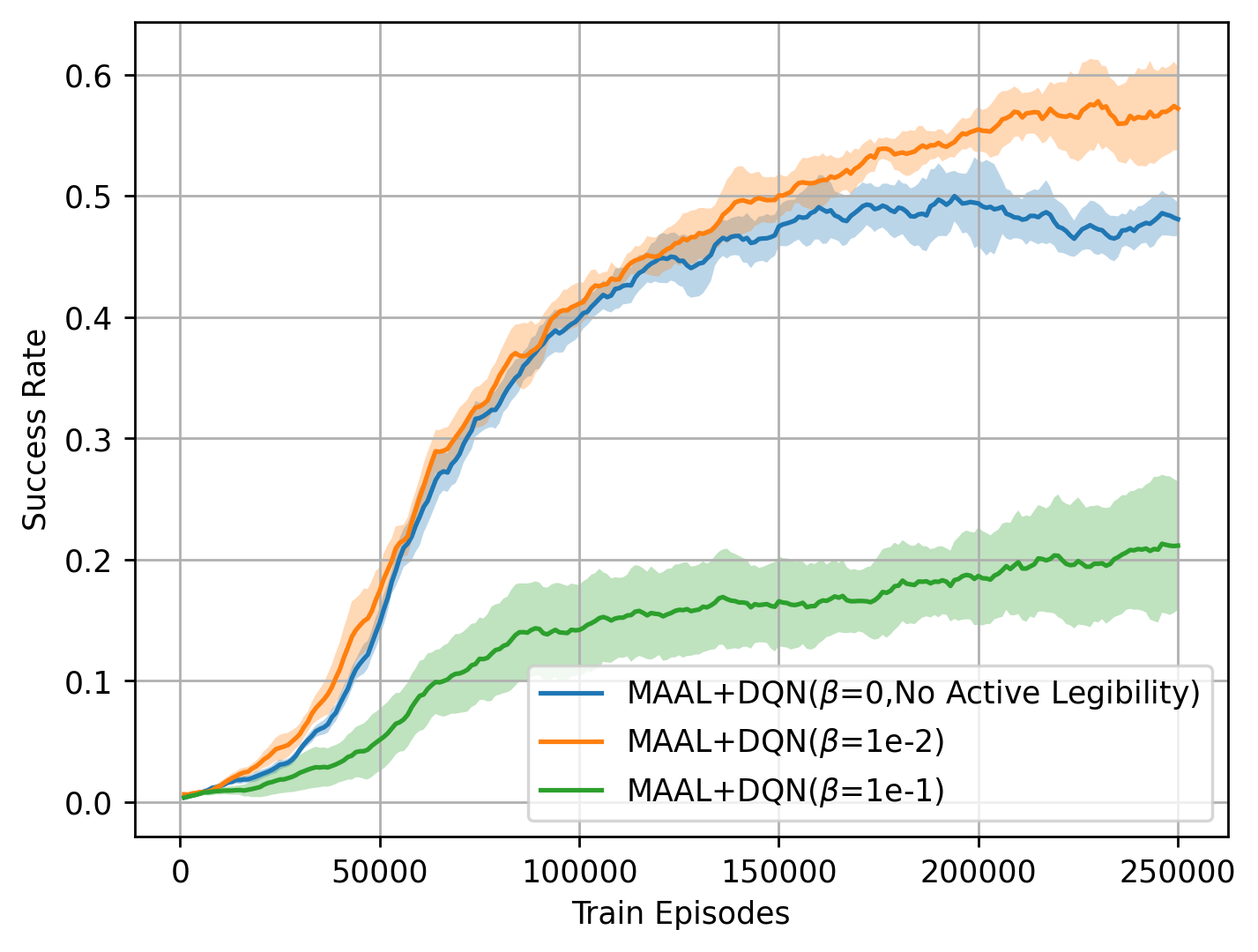}
    \label{fig:DQN success_rate}}
    \subfigure[PCR of MAAL+DQN]{\includegraphics[width=0.31\linewidth]{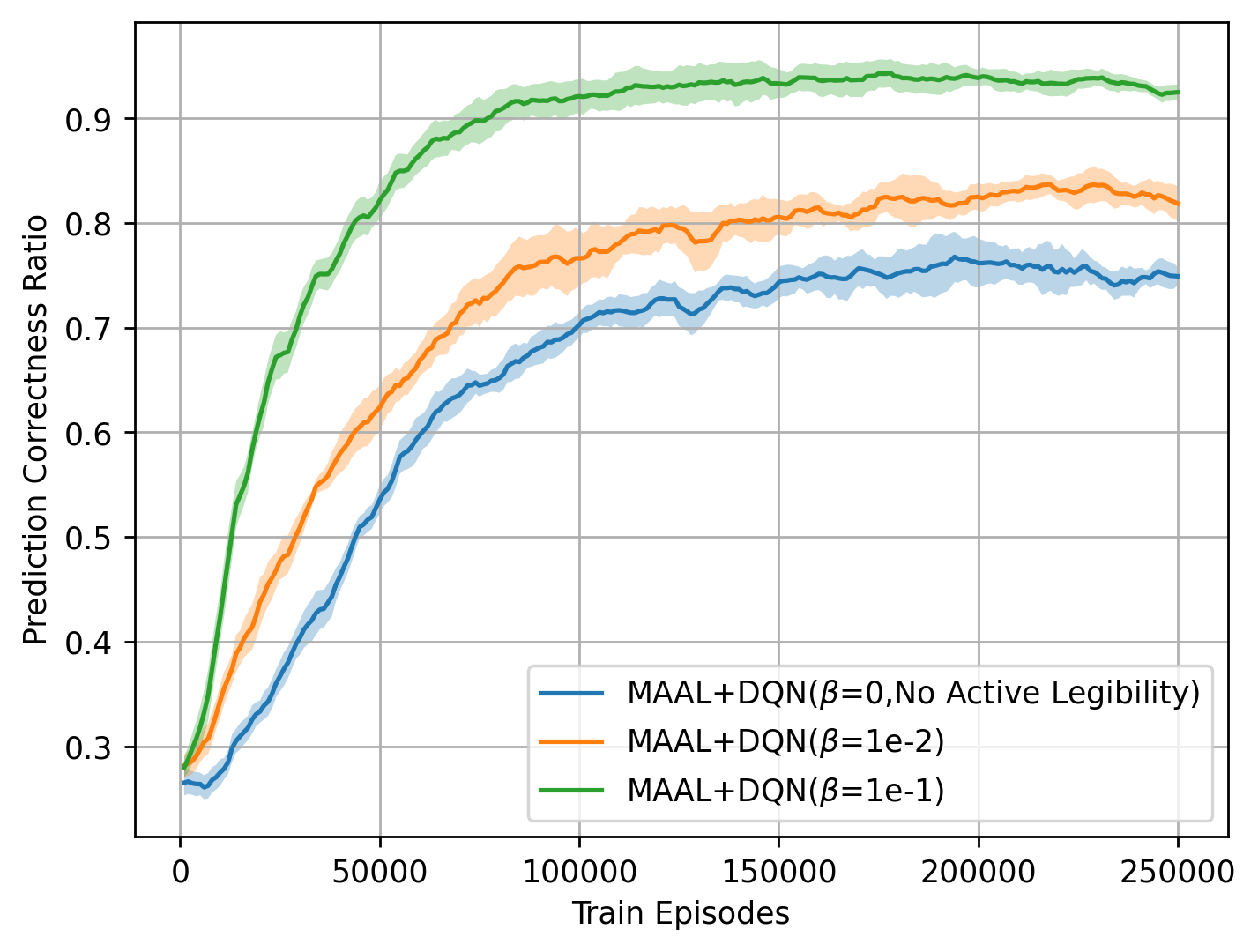}
    \label{fig:DQN Prediction Correctness Ratio}}
    \subfigure[PTR of MAAL+DQN]{\includegraphics[width=0.31\linewidth]{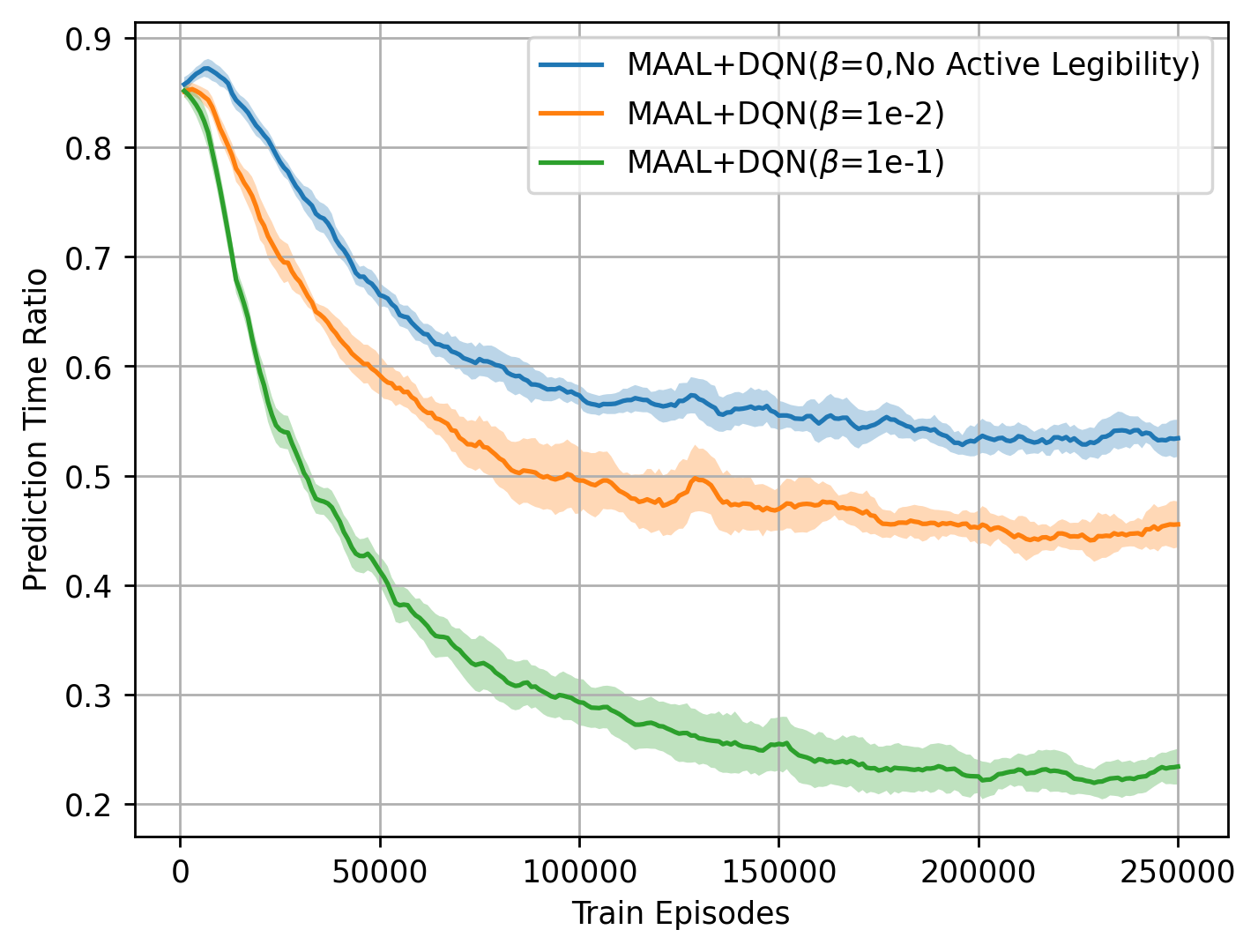}
    \label{fig:DQN Prediction Time Ratio}}
    \caption{Performance of the MAAL with DQN}
    \label{fig:Combine MAAL with and DQN}
\end{figure}

From Fig.~\ref{fig:SARSA Prediction Correctness Ratio} and Fig.~\ref{fig:SARSA Prediction Time Ratio}, we can see that legibility has a huge impact on the goal identification in SARSA. Without MAAL, the Follower can only recognize the true target of the navigator in $40\%$ episodes and obtain the correct results at almost the end of the episode. We also observe a similar pattern in MAAL+DQN from Fig.~\ref{fig:DQN Prediction Correctness Ratio} and Fig.~\ref{fig:DQN Prediction Time Ratio}. We notice that the PTR has converged to approximately $25\%$ in MAAL+QL, MAAL+SARSA, and MAAL+DQN, i.e., at least a quarter of the journey is necessary before discriminating the true goal. }

In summary, although MAAL+SARSA and MAAL+DQN do not perform as well as MAAL+QL, according to the success rate with and without MAAL, we demonstrate that extending MAAL to other RL methods is fairly feasible.

\subsection{Particle Simple Navigation}
\label{sec_Particle_Simple_Navigation}
Particle is a classic multiagent reinforcement learning environment proposed by OpenAI~\cite{mordatch2017emergence}, where motion and collisions are simulated as real rigid-body collisions, calculating velocity and displacement based on momentum and forces. Building upon Particle, we introduce a new scenario called {\it simple navigation}. Compared to LFM, the simple navigation scenario has a continuous state space with higher dimensions and more agents, making it much more complex than LFM.

\begin{figure}[h]
    \centering
\includegraphics[width=0.6\linewidth]{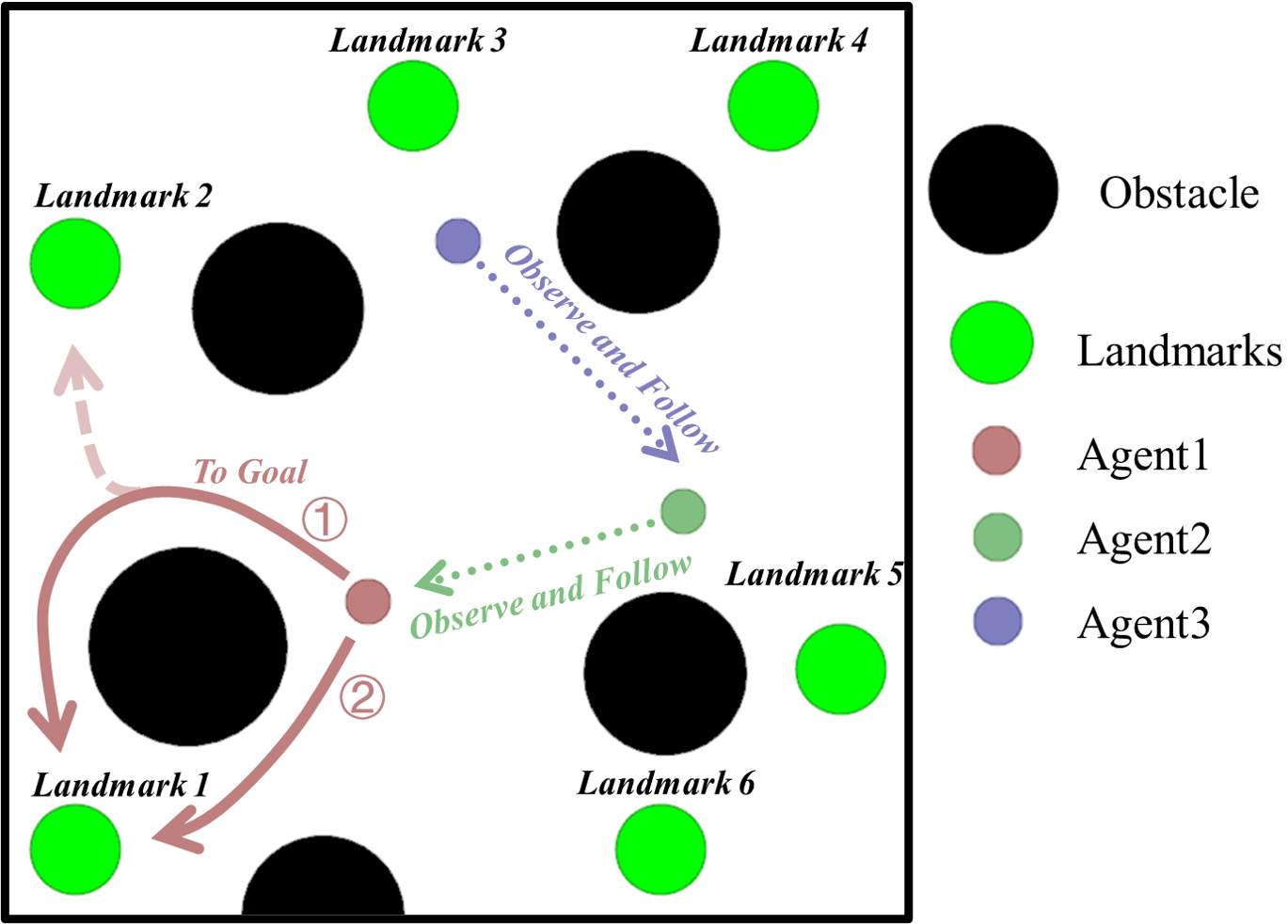}
    \caption{The Simple Navigation domain: The scenario consists of three agents, six landmarks, and several immovable circular obstacles. Similar to LFM, at the beginning of game, agent1 is assigned one landmark out of the six as target landmark, which is unknown to agents2 and agent3. The overall goal of the game is for agent1 to lead the other two agents to the target landmark. Additionally, we specify that agent2 can only observe the actions and trajectories of agent1, and agent3 can only observe the actions and trajectories of agent2, forming an observation chain.}
\label{fig: domain_simple_navi}
\end{figure}
Similar to the LFM environment , in the simple navigator environment, legibility can also help improve system performance, as illustrated in Fig.{\ref{fig: domain_simple_navi}}. In this scenario, the navigation agent, as agent1, aims to reach landmark1. Agent1 has two possible trajectories to reach the target (represented by solid red arrows), each circumventing the obstacle from a different direction. Although both trajectories can reach the target, trajectory 1 is also perceived by the follower as a potential route to landmark2. Consequently, during the first half of trajectory 1, the observer of agent1 cannot distinguish the agent's true goal from landmark1 and landmark2. In contrast, trajectory 2 does not present any ambiguity or misjudgment potential. In certain situations, such as when agents2 and agent3 are located near landmark2 in the beginning, trajectory 2 can significantly reduce the number of ineffective steps.

\subsubsection{Experiment Details}
In this experiment, to deal with high-dimensional observation space, all the agents use DQN as policy, with vectorized $[o^i\odot g^i],i=1,2,3$ as the network input, where $o^i$ is the observation vector of the agent $i$, $g^1$ is the target landmark, $\hat{g}^2$ is the agent2's prediction of agent1's target, and $\hat{g}^3$ is the agent3's prediction of agent2's target~(in one-hot encoded).
It is worth mentioning that, we used Long Short-Term Memory (LSTM)~\cite{sak2014long} for implementing plan recognition. 
LSTM is a type of recurrent neural network (RNN) capable of learning long-term dependencies and sequences of data, making it well-suited for tasks involving time-series prediction and sequence classification.
For example, agent2 uses its observations $o^2$ and the observed action  of agent1 $a^1$ to predict the distribution of agent1's target: $LSTM(O^2_{t=0}\odot a^1_{t=0},O^2_{t=1}\odot a^1_{t=1},O^2_{t=2}\odot a^1_{t=2},...)\xrightarrow{}g^2$. Similarly, agent3  uses the same kind method to observe and predict the goal of agent2.

\begin{figure}[h]
    \centering
    \includegraphics[width=0.6\linewidth]{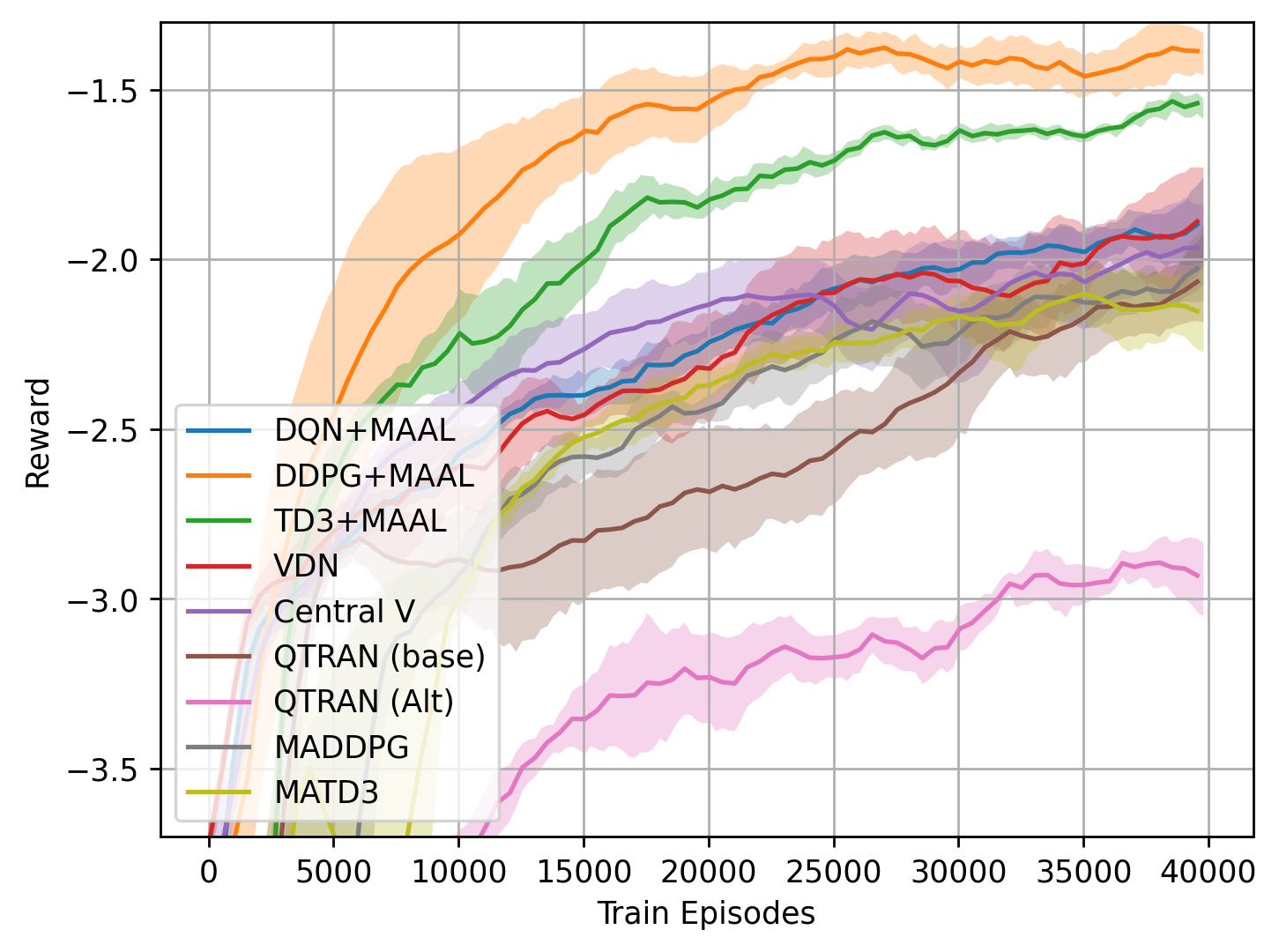}
    \caption{Episode Reward in Simple Navigation}
    \label{fig:simple_navigation_reward}
\end{figure}
The experimental results, illustrated in Fig.\ref{fig:simple_navigation_reward}, demonstrate that integrating the MAAL reward shaping with various single-agent reinforcement learning algorithms significantly enhances their performance. 
Among all the methods, DDPG+MAAL~(orange curves) achieves the highest rewards and the fastest converging speed, showing a steady and substantial improvement over the training episodes. 
The Twin Delayed Deep Deterministic Policy Gradient~(TD3)~\cite{fujimoto2018addressing} algorithm is an improved version of the DDPG algorithm, addressing some of its shortcomings by incorporating double Q-networks and delayed update mechanisms.
Notably, the performance of TD3+MAAL~(green curves) is not as well as DDPG+MAAL, likely due to the complexity introduced by TD3's double Q-networks and delayed policy updates. While these mechanisms reduce overestimation bias and improve stability in single-agent scenarios, they might hinder performance in multiagent settings where simpler and faster updates are crucial for effective coordination and interaction between agents
DQN+MAAL~(blue curves), while improved compared to its non-MAAL counterpart, does not perform as well as DDPG+MAAL or TD3+MAAL. 
Regarding other traditional MARL algorithms that do not utilize legibility and goal recognition, it can be observed that VDN (red curve), Central V (purple curve), MADDPG (gray curve), and MATD3 (yellow curve) all converge to around -2 at almost the same convergence rate. Although QTRAN-base (brown curve) also converges to a similar level, its convergence speed is significantly slower than the others. Additionally, {the alternative version of QTRAN (QTRAN-Alt, pink curves),} which constructs transformed action-value functions in different way, performs poorly, with results far below the baseline.
In conclusion, these findings demonstrate the effectiveness of the MAAL framework in enhancing cooperation and goal recognition among agents, leading to higher and more stable rewards in multiagent reinforcement learning scenarios.

\subsection{Experiment Summary}

\begin{table}[htbp]
\centering
\begin{tabular}{llll}
\hline
\textbf{Domain} & \textbf{Methods} & \textbf{Time}  & \textbf{Score}\\ \hline
\multirow{8}{*}{\makecell[{{}}]{Lead-Follow Maze  \\(250k episodes)}} & QMIX & 10.3 h & $-0.09 \pm 0.3$\\ \cline{2-4} 
 & VDN & 7.8 h & $0.49 \pm 0.46$\\ \cline{2-4} 
 & MAVEN & 17.8 h & $-0.22 \pm 0.02$ \\ \cline{2-4} 
 & IQL & 0.2 h & $0.54 \pm 0.03$ \\ \cline{2-4} 
 & MADDPG & 3.4 h & $0.00 \pm 0.00$\\ \cline{2-4} 
 & QL+MAAL & 1.3 h & {$0.89 \pm 0.05$} \\ \cline{2-4} 
 & SARSA+MAAL & 0.8 h & $-0.19 \pm 0.07$ \\ \cline{2-4} 
 & DQN+MAAL & 12.8 h & $-0.11 \pm 0.13$\\ \hline
\multirow{9}{*}{\makecell[{{}}]{Simple Navigator \\(40k episodes)}} 
 & VDN &  1.4 h & $-1.88 \pm 0.15$\\ \cline{2-4} 
 & Central V & 2.4 h & $-1.95 \pm 0.12$\\ \cline{2-4} 
 & QTRAN (base) & 3.6 h & $-2.05 \pm 0.13$\\ \cline{2-4} 
 & QTRAN (Alt) & 3.7 h & $-2.94 \pm 0.10$\\ \cline{2-4} 
 & MADDPG & 1.2 h & $-2.01 \pm 0.11$ \\ \cline{2-4} 
 & MATD3 & 1.6 h  & $-2.14 \pm 0.13$\\ \cline{2-4} 
 & DQN+MAAL & 1,3 h & $-1.88 \pm 0.12$ \\ \cline{2-4} 
 & DDPG+MAAL & 0.9 h  & $-1.39 \pm 0.06$\\ \cline{2-4} 
 & TD3+MAAL & 1.1 h  & $-1.55 \pm 0.03$\\ \hline
\end{tabular}
\caption{Summary of experimental results}
\label{tab:running_times}
\end{table}

We summarize the experimental results and running times in Table~\ref{tab:running_times}. It shows that, by enhancing the legibility of agents, we can effectively overcome the ambiguity and uncertainty in plan recognition, which allows traditional single-agent RL algorithms to become more legible, and hence significantly reducing training time costs. This is particularly evident in the lead-follow maze domain, where the policy is implemented using a Q-Table rather than a neural network, and plan recognition is achieved through Bayesian learning, resulting in lower computational complexity and faster convergence rate. 
As for the other environment, the simple navigator, MAAL does not bring significant advantages in training time due to the high computational overhead of plan recognition using LSTM. However, the MAAL algorithm still achieves the best results in the shortest time, demonstrating the applicability to problems with high-dimensional continuous state and action spaces.

\section{Conclusion and Future Work}
\label{sec_conslusion}
In this paper, we propose a Multiagent Active Legibility~(MAAL) framework to encourage agents to reveal their intentions as early as possible to achieve better collaboration in multiagent system. 
When combined with plan recognition, MAAL allows agents to utilize simple single-agent RL algorithms to achieve performance on par with, or even surpass, that of MARL methods in certain tasks.
In doing so, we employ the reward shaping technique in MAAL to make agent's actions more legible by reducing the ambiguity of its possible goals. 
Additionally, we designed two original experiments in discrete and continuous spaces: the Lead Follower Maze and Simple Navigation, both of which emphasize the speed and accuracy of goal recognition between agents. In these tests, we demonstrate the MAAL performance in those two scenarios compared to several MARL algorithms.

{
Although we have attained promising results in the environment we designed; however, the performance in other multiagent learning environments remains unknown. In future work, we will apply MAAL to solve more challenging MARL problems, such as StarCraft Multiagent Challenge~\cite{samvelyan2019starcraft}, etc. A set of benchmarks on testing the legibility's impact in MARL require more sophisticated design. It is clear that the legibility will contribute to explainable multiagent decision making from different perspectives, e.g. the agents by themselves or the observers outside the multiagent system. It would be very interesting to investigate their relations and possible unification. 
}

\appendix

\bibliographystyle{elsarticle-num} 
\bibliography{ref}




\end{document}